\documentclass[10pt]{article}
\usepackage{amsmath,amsfonts,amsthm,mathrsfs}
\usepackage{enumerate,anysize,color}
\usepackage{graphicx}
\usepackage{authblk}

\definecolor{darkred}{RGB}{100,0,0}
\definecolor{darkgreen}{RGB}{0,100,0}
\definecolor{darkblue}{RGB}{0,0,150}

\usepackage{hyperref}

\hypersetup{colorlinks=true, linkcolor=red, citecolor=blue, urlcolor=darkgreen}
\usepackage{url}

\numberwithin{equation}{section}
\newtheorem{thm}{Theorem}[section]
\newtheorem{lemma}[thm]{Lemma}
\newtheorem{pro}[thm]{Proposition}
\newtheorem{corollary}[thm]{Corollary}

\newtheorem{as}[thm]{Assumption}

\newtheorem{rem}[thm]{Remark}
\newtheorem{ex}{Example}[section]

\newcommand{\be}{\begin{equation}}
\newcommand{\ee}{\end{equation}}
\newcommand{\bea}{\begin{eqnarray*}}
\newcommand{\eea}{\end{eqnarray*}}
\newcommand{\mR}{\mathbb{R}}
\newcommand{\mN}{\mathbb{N}}
\newcommand{\mE}{\mathbb{E}}
\newcommand{\mcHK}{\mathcal{H}_K}
\newcommand{\mcE}{\mathcal{E}}

\DeclareMathOperator*{\argmin}{arg\,min}

\newcommand{\la}{\langle}
\newcommand{\ra}{\rangle}

\usepackage{anysize}

\begin{document}
\title{Iterative Regularization
%Early Stopping Regularization
for Learning with
%for Gradient Descent Learning with
Convex Loss Functions}

\renewcommand\Affilfont{\footnotesize}

\author[*]{Junhong Lin}
\author[$\dag, \circ$]{Lorenzo Rosasco}
\author[*]{Ding-Xuan Zhou}
\affil[*]{Department of Mathematics, City University of Hong Kong,
Kowloon, Hong Kong, China \authorcr
[junholin,mazhou]@cityu.edu.hk}
\affil[$\circ$]{DIBRIS, Universit\'a degli Studi di Genova, Genova 16146, Italy}
\affil[$\dag$]{LCSL, Istituto Italiano di Tecnologia and Massachusetts Institute of Technology, Cambridge, MA 02139, USA \authorcr
lrosasco@mit.edu}

\maketitle \baselineskip 16pt

\begin{abstract}
  We consider the problem of supervised  learning with convex loss functions and propose a new form of iterative regularization  based on  the subgradient method. Unlike other regularization approaches, in iterative regularization no constraint or penalization is  considered, and  generalization is achieved by (early) stopping an empirical iteration.  We consider a nonparametric setting, in the framework of reproducing kernel Hilbert spaces, and prove finite sample bounds on the excess risk under general regularity conditions. Our  study provides a  new class of efficient regularized learning algorithms and  gives  insights on the interplay between statistics and optimization in machine  learning.
\end{abstract}
%\tableofcontents

\section{Introduction}

 Availability of large high-dimensional data-sets  has  motivated
%the  development of optimization solutions for large scale learning problems.
%Among others, we mention in particular first order methods \cite{}, and their stochastic variants \cite{}, that have emerged as an interesting approach  in view of their simplicity and good performances.
%From a theoretical perspective, large scale learning has spurred novel
an interest in the interplay between statistics and optimization, towards developing new, more efficient  learning solutions \cite{Bousquet08}.  Indeed,  while much theoretical work has been  classically devoted to study  statistical properties of estimators defined by variational schemes (e.g. Empirical Risk Minimization \cite{Vapnik98} or Tikhonov regularization \cite{Tikhonov1977}),  and to  the  computational  properties of optimization procedures to solve the corresponding minimization problems (see e.g.  \cite{Sra11}), much less  work has  considered the integration of statistical and optimization aspects, see for example \cite{ChaJor13,MW14,Pistol}.

With the latter objective  in mind, in  this paper, we focus on so called iterative regularization. This class of methods,  originated in a series of work in the mid-eighties \cite{Nem86, Pol87}, is based on the observation that  early termination of an  iterative optimization scheme applied to an ill-posed problem has a regularization effect. A critical implication of this fact is that the number of iterations serves as a regularization parameter,  hence linking modeling and computational aspects: computational resources are directly linked to the generalization properties in the data, rather than their raw amount. Further, iterative regularization algorithms have a built-in  "warm restart" property which allows to compute automatically a whole sequence (path) of solutions corresponding to different levels of regularization. This latter property is especially relevant to efficiently determine the appropriate regularization via model selection.

Iterative regularization techniques are well known in the solution of inverse problems,  where several variants have been studied, see  \cite{Engl1996,KNO08} and references therein. In machine learning, iterative regularization is often simply referred to  as early stopping and is  a well known "trick", e.g. in training neural networks \cite{LeCun98}.  Theoretical studies of iterative regularization in machine learning  have mostly focused on the least squares loss function \cite{Buhlmann03,YRC07,Blanchard10,Raskutti14}. Indeed, it is in this latter case that the connection to inverse problems can be made precise \cite{lip}. Interestingly,  early stopping with the square loss has been shown to be related to boosting  \cite{Buhlmann03}  and also to be a  special case of a large class of regularization approaches based on  spectral filtering \cite{Gerfo2008,bauer}.  The regularizing effect of early stopping for loss functions other than the least squares one has hardly been studied. Indeed, to the best of our knowledge the only papers  considering  related ideas are
\cite{Bartlett07,Bickel06,Jiang04,Zhang05}, where early stopping is studied  in the context of boosting algorithms.

This paper is a different step towards understanding how early stopping can be employed with general convex loss functions. Within a statistical learning setting,  we consider convex loss functions and propose a new form of iterative regularization  based on the subgradient method,  or  the  gradient descent if the loss is  smooth. The resulting algorithms provide iterative regularization alternatives to support vector machines or regularized logistic regression, and have  built in the property of  computing the whole regularization path. Our primary contribution in this paper is  theoretical. By integrating  optimization and statistical results, we establish non-asymptotic bounds quantifying the generalization properties of the  proposed method under standard regularity assumptions. Interestingly, our study shows that considering  the last iterate leads to  essentially the same results as considering
averaging, or selecting of  the "best" iterate, as typically done in subgradient methods \cite{boyd}. From a technical point of view, considering a general convex loss  requires  different error decompositions than those for the square loss. Moreover, operator theoretic techniques need to be
replaced by convex analysis and empirical process  results. The error decomposition we consider, accounts for the contribution of  both optimization and statistics to the error, and could be useful also for other methods.

The rest of the paper is organized as as follows. We begin in Section \ref{LearningAlg} by briefly recalling the supervised learning problem, and then introduce our learning algorithm, discuss  its numerical realization.
In Section \ref{MainAnalysis}, after discussing the assumptions that underlie our analysis, we present our main theorems with discussions and discuss the general error decomposition which are composed of three error terms: the computational, the sample and approximation error terms. In Section \ref{ProofResultConvexLoss}, we will estimate computational error, while
in Section \ref{sectionSampleError}, we develop sample error bounds, and finally prove our main results.

\section{Learning Algorithm}\label{LearningAlg}
After briefly recalling  the supervised learning problem, we introduce the algorithm we propose and give some comments on its numerical realization.

\subsection{Problem Statement}\label{SubsectionProbSta}

In this paper we consider the problem of supervised learning.
Let $X$ be a separable metric space, $Y\subseteq \mR$ and let $\rho$ be a Borel probability
measure on $Z=X\times Y.$ Moreover, let  $V: \mR \times \mR \to \mR_{+}$ be a so called loss function,  measuring  the
{\em local} error $V(y, f(x))$ for $(x, y) \in Z$ and  $f: X \to \mR$. The \emph{generalization error (or expected risk)} $\mathcal{E} = \mathcal{E}^{V}$ associated to  $V$ is given by
$$\mathcal{E} (f)=\int_Z V(y, f(x)) d\rho,$$ and is well defined for any measurable loss function $V$ and measurable function $f$.
We  assume throughout that there exists a function  $f_\rho^V$ that minimizes  the expected error $\mathcal{E}(f)$
among all measurable functions $f: X \to Y$. Roughly speaking,  the goal of learning is to find
an approximation  of $f_\rho^V$ when the measure  $\rho$ is known only through
 a sample $\mathbf z=\{z_i=(x_i, y_i)\}_{i=1}^m$ of size $m\in\mN$ independently and identically drawn according to $\rho$.
More precisely, given $\mathbf z$ the goal is to design a computational procedure to efficiently estimate
a function $f_\mathbf z$, an estimator, for which it is possible to derive an explicit probabilistic
bound on  the excess expected risk
$${\mathcal E}(f_{\mathbf z}) - {\mathcal E}(f_\rho^V).
$$
%In particular, we are interested in the
%exponential bound on
%$$
% {\mathbb P}\left( {\mathcal E}(f_\mathbf z)- {\mathcal E}(f_\rho^V)\ge \epsilon \right)
%$$
%for all $\epsilon >0$ and where ${\mathbb P}=\rho^n$. Further, a goal will be use the latter bound
%to establish the
%
%strong consistency of the estimator, that is
%$$
% {\mathbb P}\left( \lim_{m\to \infty}{\mathcal E}(f_\mathbf z)= {\mathcal E}(f_\rho^V) \right)=1.
%$$
We end with this section with a remark and an example.
\begin{rem}
For several loss functions, it is possible to show that $f_\rho^V$ exists-- see example below. However, as will be seen in the following,
the search for an estimator in practice is often restricted to some hypothesis space
$\mathcal H$ of measurable functions. In this case  one should replace ${\mathcal E}(f_\rho^V)$ by $\inf_{f\in \mathcal H} {\mathcal E}(f)$.
Interestingly, examples of  hypothesis spaces are known for which    ${\mathcal E}(f_\rho^V)=\inf_{f\in \mathcal H} {\mathcal E}(f)$, namely universal hypothesis spaces \cite{SteinwartChristmann2008a}.  In the following, we consider ${\mathcal E}(f_\rho^V)$, with the understanding that it should be replaced by the infimum over $\mathcal H$, if the latter is not universal.
\end{rem}
\noindent The following example gives several possible choices of loss functions.
\begin{ex}\label{ExampleLoss}
The most classical example of loss function is probably  the square loss $V(y,a)=(y-a)^2$, $y,a\in \mR$. In this case, $f_\rho^V$ is the regression function,
defined at every point as  the expectation of the conditional distribution of $y$ given $x$ \cite{CuckerZhou,SteinwartChristmann2008a}. Further examples include the absolute value loss $V(y,a)=|y-a|$
for which $f_\rho^V$ is the median of the conditional distribution and more generally $p$-loss functions $V(y,a)=|y-a|^p$, $p \in \mN$.
Vapnik's $\epsilon$-insensitive loss $V(y,a)=\max\{|y-a|-\epsilon,0\}$, $\epsilon>0$
and its generalizations  $V(y,a)=\max\{|y-a|^p-\epsilon,0\}$, $\epsilon>0, p>1$ provide yet other examples.
For classification i.e. $Y=\{\pm 1\}$, other examples of loss functions used in classification,  include the hinge loss  $V(y,a)=\max\{1-ya,0\}$ , the logistic loss
 $V(y,a)=\log(1+e^{-ya})$  and the exponential  loss $V(y,a)=e^{-ya}$. For all these examples $f_\rho^V$ can be computed,  see e.g. \cite{SteinwartChristmann2008a}, and measurability is easy to check.
\end{ex}

\subsection{Learning via Subgradient Methods with Early Stopping}

 To  present  the proposed learning algorithm we need a few preliminary definitions.
 Consider a reproducing kernel $K: X \times X \to \mR$, that is a symmetric function, such that the matrix $(K(u_i, u_j))_{i, j=1}^\ell$
is positive semidefinite for any finite set of points $\{u_i\}_{i=1}^\ell$ in $X$. Recall that
a reproducing kernel $K$ defines a reproducing kernel Hilbert space (RKHS)
$(\mathcal{H}_K, \|\cdot\|_K)$ as the
completion of the linear span of the set
$\{K_x(\cdot):=K(x,\cdot): x\in X\}$ with respect to the inner
product $\la K_x, K_u\ra_{K}:=K(x,u) \cite{aronszajn50}$. Moreover,  assume the loss function $V$ to be measurable and convex in its second argument, so that the corresponding left derivative $V_-'$ exists
and is  non-decreasing at every point. For  a step size sequence $\{\eta_{t}>0\}$, a stopping iteration  $T >2$ and a initial value
 $f_1 =0$,  we consider the iteration
\be\label{Alg}
f_{t+1}=f_t-\eta_t\frac{1}{m}\sum_{j=1}^{m} V'_- (y_j, f_t(x_j)) K_{x_j}, \qquad t=1, \ldots, T.
\ee
The above iteration  corresponds to the subgradient  method \cite{Bertsekas99,Boyd03} for minimizing  the
\emph{empirical error} $\mcE_{\mathbf{z}} = \mcE_{\mathbf{z}}^{V}$ with respect to the loss $V$, which is given
by
$$\mcE_{\mathbf{z}} (f)=\frac{1}{m}\sum_{j=1}^{m}V(y_j, f(x_j)). $$
Indeed, it is easy to see that $\frac{1}{m}\sum_{j=1}^{m} V'_- (y_j, f(x_j)) K_{x_j}\in \partial \mcE_{\mathbf{z}} (f)$,
the subgradient of the empirical risk for  $f\in \mathcal H _K$. In the special case where the loss function is smooth then \eqref{Alg} reduces to the
gradient descent algorithm. Since the subgradient method is not a descent algorithm,  rather then the last iterate,  the so called Ces\'aro mean is often considered,  corresponding, for $T \in \mN$, to the following weighted average
\begin{equation}\label{eq:aveite}
a_T = { \sum_{t=1}^{T} \omega_t f_t}, \quad\quad \omega_t = {\eta_t \over \sum_{t=1}^T \eta_t}, \quad t=1, \dots, T.
\end{equation}
Alternatively, the {\em best} iterate is also often considered, which is defined for $T \in \mN$ by
\begin{equation}\label{eq:bestite}
b_T = \argmin_{t=1,\cdots,T} \mathcal{E}_{\bf z} (f_t).
\end{equation}
 In what follows, we will  consider the learning algorithms obtained considering these different choices.

We note that, classical results \cite{Bertsekas99,Boyd03,boyd} on the subgradient method focus on  how the iteration~\eqref{Alg}
can be used to minimize $\mcE_{\mathbf{z}}$.  Different to these studies,  in the following  we are interested in showing how iteration~\eqref{Alg} can be used to define a {\em statistical estimator}, hence a learning algorithm to  minimize the expected risk $\mathcal E$, rather than  the empirical risk $\mcE_{\mathbf{z}}$.
We end with one remark.
\begin{rem}[Early Stopping SVM and Kernel Perceptron]
If we consider the hinge loss function in~\eqref{Alg},  the corresponding algorithm  is closely related to a batch (kernel) version of the perceptron \cite{perceptronr,perceptron},
where an entire pass over the data is done before updating the solution. Such an algorithm
can also be  seen as an early stopping version of Support Vector Machines \cite{SVM}.  Interestingly, in this case the whole regularization path is
computed incrementally albeit sparsity could be lost. We defer to a future work  the study of the practical implications of these observations.
\end{rem}

\subsection{Numerical Realization}
The simplest case to derive a numerical procedure from Algorithm \ref{Alg} is when $X=\mR^d$ for some $d\in \mN$ and
$K$ is the associated  inner product. In this case it is straightforward to see that $f_{t+1}(x)= w_{t+1}^\top x$ for all $x\in X$, with
$$
w_{t+1}=w_t-\eta_t\frac{1}{m}\sum_{j=1}^{m} V'_- (y_j, w_t^\top x_j ) x_j, \qquad t=1, \ldots, T,
$$
and $w_1=0$.

Beyond the linear kernel, it can be easily seen that  given  a finite {\em dictionary}
$$\{\phi_i:X\to \mR, i=1,\dots, p\},\qquad p\in \mN,$$
one can consider the kernel $K(x,x')=\sum_{i=1}^p  \phi_i(x')\phi_i(x)$. In this case, it holds  $f_{t+1}(x)=\sum_{i=1}^p w^i_{t+1} \phi_i(x)= w_{t+1}^\top\Phi(x)$, $\Phi(x)=(\phi_1(x), \dots, \phi_p(x))$
for all $x\in X$,  with
$$
w_{t+1}=w_t-\eta_t\frac{1}{m}\sum_{j=1}^{m} V'_- (y_j,  w_{t+1}^\top \Phi(x_j)  )\Phi(x_j), \qquad t=1, \ldots, T,
$$
and $w_1=0$.
Finally, for a general kernel  it is easy to prove by induction that $f_{t+1}(x)=\sum_{j=1}^m c_{t+1}^j K(x,x_j)$ for all $x\in X$, with
$$
c_{t+1}=c_t-\eta_t\frac{1}{m}g_t, \qquad t=1, \ldots, T,
$$
for $c_1=0$ and $g_t\in \mR^m$ with $g_t^i= V'_- (y_i, \sum_{j=1}^m c_{t}^j K(x_i,x_j))$. Indeed, The base case is straightforward to check and  moreover by the inductive hypothesis
$$
f_{t+1}=\sum_{j=1}^m c_{t}^j K_{x_j} -\eta_t\frac{1}{m}\sum_{j=1}^{m} V'_- (y_j, f_t(x_j)) K_{x_j}=
\sum_{j=1}^m  K_{x_j} \left(  c_{t}^j-\eta_t\frac{1}{m} V'_- (y_j, f_t(x_j)) \right).
$$
%\begin{figure}[htbp]
% \centering
% \includegraphics[width=6in]{linearKernelSG.eps}
%   \caption{SG}
%\end{figure}

\section{Main Results with Discussions}\label{MainAnalysis}
After presenting our main assumptions, in this section we state and discuss our main results.

\subsection{Assumptions}
Our learning rates will be stated under several conditions on the triple $(\rho, V, K)$, that we describe and comment next.
We begin with a basic assumption.
\begin{as}\label{Boundness}
We assume the kernel to be bounded, that is
$\kappa=\sup_{x\in X}\sqrt{K(x,x)}<\infty$ and  moreover   $\|f_{\rho}^V\|_{\infty}<\infty$ and  $|V|_0 :=\sup_{y\in Y} V(y, 0) < \infty$.
Furthermore, we consider the following  growth condition for the left derivative $V'_- (y, \cdot)$. For some $q \geq 0$ and constant $c_q>0,$ it holds,
\be\label{EqCond1} \left|V'_- (y, a)\right| \leq c_q (1+|a|^q), \qquad \forall a\in \mR, y\in Y.\ee
\end{as}
\noindent The  boundness conditions on $K, f_{\rho}^V$ and $V$ are fairly common \cite{CuckerZhou,SteinwartChristmann2008a}. They could probably be weakened by considering a more  involved analysis which is outside the scope of this paper. Interestingly,  the {\em growth} condition on the left derivative of $V$ is weaker than assuming the loss,   or its gradient, to be  Lipschitz in its second entry, as   often done   both in learning theory  \cite{CuckerZhou, SteinwartChristmann2008a} and in optimization \cite{boyd}.
We note that the growth condition~\eqref{EqCond1} is implied by the requirement for the loss function to be Nemitiski, as introduced in \cite{rep} (see also \cite{SteinwartChristmann2008a}). This latter condition, which is satisfied by most loss function, is natural to provide a variational characterization of the learning problem.

The second assumption refines the above boundness condition by considering  a  variance-expectation bound
which quantifies   a notion of noise in the  measure $\rho$ with respect to the balls
$B_R =\left\{f\in {\mathcal H}_K: \|f\|_K \leq R\right\}$ in ${\mathcal H}_K$, \cite{CuckerZhou,SteinwartChristmann2008a}.

\begin{as}\label{assumption3} We assume that there exist an exponent $\tau
\in [0, 1]$ and a positive constant $c_\tau$ such that for any $R \geq 1$ and $f\in B_R$, we have
\begin{eqnarray}
\int_{\mathcal Z} \left\{\left(V(y, f(x)) - V(y,
f_\rho^V(x)\right)^2\right\} d \rho \leq c_\tau R^{2 + q-\tau}
\left\{{\mathcal E} (f)
-{\mathcal E} (f_\rho^V)\right\}^\tau.  \label{varianceexpect}
\end{eqnarray}
\end{as}

\noindent Assumption \ref{assumption3} always holds true for $\tau =0$, in which case  $c_\tau$
will also depend on $\|f_{\rho}^V\|_{\infty}$. In classification, the above condition can be related to the so called
Tsybakov margin condition. The latter quantifies the intuition that a classification problem is hard if the conditional probability of $y$
given $x$ is close to $1/2$ for many input points. More precisely if we  denote by $\rho(y|x)$ the conditional probability  for all
$(x,y)\in Z$ and by $\rho_X$ the marginal probability on $X$, then we say that $\rho$ satisfies the Tsybakov margin condition
with exponent $s$ if there exists a constant $C>0$ such that for all $\delta>0$
$$
\rho_X(\{x\in X ~:~ |\rho(1|x)-\frac 1 2|\le \delta\})\le (C\delta)^{s}.
$$
Interestingly, under Tsybakov margin  condition  Assumption \ref{assumption3} holds with $\tau=\frac{s}{s+1}$ and with $c_\tau$ depending only $C$.
% note that ingo say it also hold if s=0

%If the triple $(\rho, V, K)$ satisfies some conditions, the exponent $\tau$
%can be larger.

The third condition is  about the  decay of a suitable notion approximation error \cite{SmaleZhouapprox}.
\begin{as}\label{ApproDef}
Let $ f_{\lambda}$ be a minimizer of:
\begin{equation}\label{targetf}
f_{\lambda}:=\argmin_{f\in \mathcal{H}_K} \mathcal{E}(f_{\lambda}) + \lambda \|f\|_K^2.
\end{equation}
The approximation error associated with the tripe $(\rho, V, K)$ is defined by
\begin{equation}\label{approxerror}
\mathcal{D}(\lambda) = \mathcal{E}(f_{\lambda}) - \mathcal{E}(f_{\rho}^V) + \lambda \|f\|_K^2.
\end{equation}
We assume that for some $\beta \in (0,1]$ and $c_{\beta}>0$, the approximation error satisfies
\be
\mathcal{D}(\lambda) \leq c_{\beta}\lambda^{\beta}, \qquad \forall \ \lambda> 0.
\label{decayapprox}
\ee
\end{as}
\noindent
The above assumption is standard when analyzing regularized empirical risk minimization and is related to the
the definition of interpolation spaces by means of so the called $K$- functional \cite{CuckerZhou}.
 Interestingly, we will see in the following that it is also important when analyzing the approximation properties of the subgradient algorithms \ref{Alg}.

% put how it is implied by the source condition

Finally, the last condition characterizes the {\em capacity} of a ball in the  RKHS  $\mathcal{H}_K$ in terms of empirical covering numbers, and plays an essential role in sample error estimates. Recall that for a subset ${\cal G}$ of a metric space $(H,d)$, the covering number $\mathcal{N}({\cal G},\epsilon, d)$ is defined by
$$ \mathcal{N}({\cal G},\epsilon, d) = \inf\left\{ l\in \mN: \exists  f_1, f_2, \cdots, f_l \subset H \mbox{ such that } {\cal G} \subset \bigcup_{i=1}^\ell \{f\in{\cal G}:
d(f,f_i) \le \epsilon\} \right\}.$$

\begin{as}\label{covnum} Let ${\cal G}$ be a set of functions on
$X$.  The metric $d_{2,\mathbf{z}}$ is defined on ${\cal G}$ by
$$d_{2,{\mathbf z}} (f,g) = \left\{ {1 \over m} \sum_{i=1}^{m} (f(z_i) - g(z_i))^2 \right\}^{1/2}, \quad f,g\in {\cal G}.$$
We assume that for some $\zeta \in (0, 2)$, $c_\zeta >0$, the covering numbers of the unit ball $B_1$ in ${\mathcal H}_K$ with respect to $d_{2, {\mathbf z}}$ satisfy
\begin{equation}\label{capacityB}
\mE_{\bf z} \left[\log {\cal N}\left(B_1, \epsilon, d_{2,{\bf z}} \right) \right)] \leq c_\zeta \left(\frac{1}{\epsilon}\right)^\zeta, \qquad \forall~ \epsilon >0.
\end{equation}
\end{as}
%check X o Z?
The smaller is $\zeta$ the more stringent is the capacity assumption.
 As $\zeta$ approaches $2$ we are essentially considering a capacity independent scenario.
In what follows, we will briefly comment on the connection between the above assumption and other related assumptions.
Recall that capacity of the RKHS may be measured by various concepts: covering numbers of balls $B_R$ in ${\mathcal H}_K$, (dyadic) entropy numbers and decay of the  eigenvalues of the integral operator $L_{K}: L^2_\rho\to L^2_\rho$ given by $L_{K}(f) = \int_{X} f(x) K_x d\rho(x),$ where $L^2_\rho=\{f:X\to {\mathbb R}~:~ \int |f(x)|^2d\rho \}.$
For a subset ${\cal G}$ of a metric space $(H,d)$, its $n$-th entropy number is defined by
$$ e_n({\cal G}, d) = \inf \left\{ \varepsilon>0: \exists f_1,f_2,\cdots,f_{2^{n-1}} \mbox{ such that } {\cal G} \subset \bigcup_{i=1}^{2^{n-1}} \{f\in{\cal G}:
d(f,f_i) \le \varepsilon\} \right\}.$$
First, the covering and entropy numbers are equivalent (see e.g. \cite[Lemma 6.21]{SteinwartChristmann2008a}). Indeed,  for $\zeta > 0,$ the covering numbers $\mathcal{N}({\cal G},\epsilon, d)$ satisfy
$$ \log \mathcal{N}({\cal G},\epsilon, d) \leq a_{\zeta} \left( {1 \over \epsilon} \right)^{\zeta}, \qquad \forall \epsilon>0,$$
for some $a_{\zeta} > 0$, if and only if the entropy numbers $e_n({\cal G}, d)$ satisfy
$$ e_n({\cal G}, d) \leq a_{\zeta}' \left( {1 \over n} \right)^{1 \over \zeta},$$
for some $a_{\zeta}'>0.$
Second, it is shown in \cite{Steinwart2009} that if the eigenvalues of the integral operator $L_{K}$ satisfy
$$\lambda_{n} \leq \tilde{a}_{\zeta} \left({1\over n}\right)^{2 \over \zeta} \qquad n \geq 1$$
for some constants $\tilde{a}_{\zeta} \geq 1$ and $\zeta \in (0,2),$
then the expectations of the random entropy numbers $\mE_{\bf z}[e_n \left(B_1, d_{2,{\bf z}} \right)]$ satisfy
$$ \mE_{\bf z}[e_n \left(B_1, d_{2,{\bf z}} \right)] \leq a_{\zeta} \left({1 \over n} \right)^{1 \over \zeta},\qquad n\geq 1 $$
for some constant $a_{\zeta}.$ Hence, using the equivalence of  covering and entropy numbers,
 $\mE_{\bf z} \left[\log {\cal N}\left(B_1, \epsilon, d_{2,{\bf z}} \right) \right)]$ one can be estimated from the
 eigenvalue decay  of the integral operator $L_{K}$. Last, since $d_{2,{\bf z}}(f,g) \leq \|f - g\|_{\infty},$
one has that for any $\epsilon>0,$ ${\cal N}\left(B_1, \epsilon, d_{2,{\bf z}} \right)$ is bounded by ${\cal N} \left(B_1, \epsilon, \|\cdot\|_{\infty}\right),$ the uniform covering number of $B_1$ under the metric $\|\cdot\|_{\infty},$
Thus, the covering numbers ${\cal N}\left(B_1, \epsilon, d_{2,{\bf z}} \right)$ can be estimated given  the uniform smoothness of the kernel \cite{Zhou03}.

\subsection{Finite Sample Bounds for General Convex Loss Functions}

The following is our main  result providing a general finite sample bound for the iterative regularization
induced by the the subgradient method for  convex loss functions considering the last iterate.

\begin{thm}\label{MainRates}
%Under Assumptions~\eqref{Boundness}\eqref{assumption3}\eqref{ApproDef}\eqref{covnum},
Assume (\ref{EqCond1}) with $q \geq 0$, (\ref{varianceexpect}) with  $\tau\in[0,1]$, (\ref{decayapprox}) with $\beta \in (0,1]$ and (\ref{capacityB}) with $\zeta\in(0,2)$,.
Let $\eta_t=\eta_1 t^{-\theta}$ with $0< \theta <1$ satisfying $\theta > \frac{q}{q+1}$ and $\eta_1$ satisfying
\be\label{restr1} 0< \eta_1\leq \min\left\{\frac{\sqrt{1-\theta}}{\sqrt{2} c_q (\kappa +1)^{q+1}}, \frac{1-\theta}{4|V|_0}\right\}. \ee
If $T$ is the integer part of $\lceil m^{\gamma} \rceil$, then for any $0< \delta <1$, with confidence $1-\delta$, we have
$$
{\mathcal E}(f_T) - {\mathcal E}(f_\rho^V)
\leq \left\{\begin{array}{ll}
\widetilde{C} m^{-\alpha} \log \frac{2}{\delta}, & \hbox{when} \ \theta > \frac{q+1}{q+2}, \\
\widetilde{C} m^{-\alpha} \log m  \log \frac{2}{\delta}, & \hbox{when} \ \theta \leq \frac{q+1}{q+2}, \end{array}\right.
$$
where the power indices $\gamma$ and $\alpha$ are defined as
\begin{eqnarray}
\gamma &=& \left\{\begin{array}{ll}
\frac{2}{1-\theta} \frac{1}{\left(1 + 2 \beta\right) (2-\tau+\zeta \tau/2) + q (1 + \zeta/2)}, & \hbox{when} \ \theta \geq \frac{q+1}{q+2}, \\
\frac{2}{1-\theta} \frac{1}{\left(1 + \frac{2 \beta (\theta(1+q) -q)}{1-\theta}\right) (2-\tau+\zeta \tau/2) + q (1 + \zeta/2)}, & \hbox{when} \ \theta < \frac{q+1}{q+2},  \end{array}\right. \label{indexgamma} \\
\alpha &=& \left\{\begin{array}{ll}
\frac{\beta}{\beta (2 -\tau + \zeta \tau/2) + \left\{\frac{2 -\tau + \zeta \tau/2}{2} + \frac{q(1+ \zeta/2)}{2}\right\}}, & \hbox{when} \ \theta \geq \frac{q+1}{q+2}, \\
\frac{\beta}{\beta (2 -\tau + \zeta \tau/2) +  \frac{1-\theta}{\theta(1+q) -q} \left\{\frac{2 -\tau + \zeta \tau/2}{2} + \frac{q(1+ \zeta/2)}{2}\right\}}, & \hbox{when} \ \theta < \frac{q+1}{q+2}, \end{array}\right. \label{indexalpha}
\end{eqnarray}
and $\widetilde{C}$ is a constant independent of $m$ or $\delta$ (given explicitly in the proof).
\end{thm}
The proof is deferred to Section~\ref{sectionSampleError} and is based on a   novel error decomposition,
discussed in Section~\ref{sec:errdec}, integrating statistical and optimization aspects.
We illustrate the above result  for Lipschitz loss functions, that is considering $q= 0$.

%HERE

%Our proof for the above results relies on a key error decomposition, and consists of two steps: estimation of the computational error and bounding sample error. While our probabilistic upper bounds seem suboptimal, we remark that our proof for the computational error, namely, proof of Lemma \ref{ExcessEmp}, is totally new and can be also used in a general convex optimization problem.
%
%By setting $q= 0$ in Theorem \ref{MainRates}, we get the following results for Lipschitz loss.
\begin{corollary}\label{MainRatesLipschitz}
Assume (\ref{EqCond1}) with $q = 0$,  (\ref{capacityB}) with $\zeta\in(0,2)$ and (\ref{decayapprox}) with $\beta \in (0,1]$.
Let $\eta_t=\eta_1 t^{-\theta}$ with $0< \theta <1$ and $\eta_1$ satisfying
$ 0< \eta_1\leq \min\left\{\frac{\sqrt{1-\theta}}{\sqrt{2} c_q (\kappa +1)}, \frac{1-\theta}{4|V|_0}\right\}. $
If $T$ is the integer part of $\lceil m^{\gamma} \rceil$, then for any $0< \delta <1$, with confidence $1-\delta$, we have
$$
{\mathcal E}(f_T) - {\mathcal E}(f_\rho^V)
\leq \left\{\begin{array}{ll}
\widetilde{C} m^{-\alpha} \log \frac{2}{\delta}, & \hbox{when} \ \theta > \frac{1}{2}, \\
\widetilde{C} m^{-\alpha} \log m  \log \frac{2}{\delta}, & \hbox{when} \ \theta \leq \frac{1}{2}, \end{array}\right.
$$
where the power indices $\gamma$ and $\alpha$ are defined as
\begin{eqnarray}
\gamma &=& \left\{\begin{array}{ll}
\frac{2 }{(1-\theta)(2 \beta + 1)(2-\tau+\zeta \tau/2)}, & \hbox{when} \ \theta \geq \frac{1}{2}, \\
\frac{2}{(1-\theta  + 2 \beta \theta) (2-\tau+\zeta \tau/2)}, & \hbox{when} \ \theta < \frac{1}{2},  \end{array}\right. \nonumber\label{indexgammaA} \\
\alpha &=& \left\{\begin{array}{ll}
\frac{2\beta}{(2 \beta + 1) (2-\tau+\zeta \tau/2)}, & \hbox{when} \ \theta \geq \frac{1}{2}, \\
\frac{2\theta \beta}{ (1-\theta + 2\beta \theta) (2-\tau+\zeta \tau/2)}, & \hbox{when} \ \theta < \frac{1}{2}, \end{array}\right.\nonumber \label{indexalphaA}
\end{eqnarray}
and $\widetilde{C}$ is a constant independent of $m$ or $\delta$.
\end{corollary}

The above results give finite sample bounds on the excess risk, provided that a suitable
stopping rule is considered. While the stopping rule in above theorems is  distribution dependent, a  data-drive stopping rule can be given by hold-out cross validation and adaptively achieves the same bound. The proof of this latter result is straightforward using the techniques in \cite{capyao10} and is omitted. The obtained bounds  directly yields strong consistency (almost sure convergence) using standard arguments. Interestingly,  our analysis suggests that a decaying stepsize needs to be chosen to achieve meaningful error bounds. The stepsize choice can influence both the  early stopping rule and the error bounds. More precisely, if the step size decreases fast enough $\theta\ge \frac{q+1}{q+2}$,   the stopping rule depends on the decay speed but the error bound does not. In this case the best possible choice for the early stopping rule  is $\theta= \frac{q+1}{q+2}$, that is $\eta_t\sim 1/\sqrt{t}$ in the case of Lipschitz loss functions. With this choice, if for example we take the limit $\beta\to1$, $\tau \to0$,  we have that the stopping rule scales as  $O(m^{2/3})$ whereas the corresponding finite sample bounds is $O(m^{-1/ 3})$.
A  slower stepsize decay  given by   $\theta<\frac{q+1}{q+2}$ affects both  the stopping rule and the error bounds, but the results in these regime worsen.
A more detailed discussion of the obtained bounds in comparison to other learning algorithms is postponed to Section~\ref{sec:errcomp}.
%We see from the above corollary that for $\tau = 0$ and Lipschitz loss, the upper bound is depending on parameters $\theta$ and $\beta$. Although the capacity parameter $\zeta$ is not directly related to the upper bound in this case, the capacity assumption with $\zeta$ is implicitly used in our proofs, and the error term related to $\zeta$ is absorbed by other error terms since $\zeta <2$. %Note that for a general RKHS, the capacity condition (\ref{capacityB}) is always satisfied with some $\zeta\in(0,2)$ \cite{}.
%In the idea case $f_{\rho}^V \in \mathcal{H}_K$ which implies $\|f_{\rho}^V\|_K < \infty$ and $\beta \to 0$, by choosing $\theta = 1/2$ and $\tau = 0$ in Corollary \ref{MainRatesLipschitz}, we see that the convergence rate and the stopping rule are $O(m^{-1/ 3})$ and $O(m^{2/3})$ respectively. The convergence rate and the stopping rule seem suboptimal and it would be interesting to improve them with a further developed proof technique in the future.
Next we discuss the behavior of different variants of the proposed algorithm.

As mentioned before  in the subgradient method,  when the goal is empirical risk minimization,  the average or best   iterates are often considered
(see \eqref{eq:aveite}, \eqref{eq:bestite}). It is natural to ask what are the properties of the estimator obtained with these latter choices, that is when they are  used as approximate minimizers of the expected, rather than the empirical, risk.
The following theorem provides an answer.
\begin{thm}\label{MainRatesAverge}
  Under the assumptions of Theorem \ref{MainRates}, if $T$ is the integer part of $\lceil m^{\gamma} \rceil$ and $g_T = a_T$ (or $b_T$) then for any $0< \delta <1$, with confidence $1-\delta$, we have
$$
{\mathcal E}(g_T) - {\mathcal E}(f_\rho^V)
\leq \left\{\begin{array}{ll}
\overline{C} m^{-\alpha} \log \frac{2}{\delta}, & \hbox{when} \ \theta \neq \frac{q+1}{q+2}, \\
\overline{C} m^{-\alpha} \log m  \log \frac{2}{\delta}, & \hbox{when} \ \theta = \frac{q+1}{q+2}, \end{array}\right.
$$
where the power indices $\gamma$ and $\alpha$ are defined as in Theorem \ref{MainRates}
and $\overline{C}$ is a constant independent of $m$ or $\delta$ (can be given explicitly).
  \end{thm}
The above result shows  that, perhaps surprisingly,   the behavior of the average and best iterates is   essentially the same  as the last iterate.  Indeed,  there is only a subtle difference between the upper bounds in Theorem \ref{MainRatesAverge} and Theorem \ref{MainRates}, since the latter  has an extra $\log m$ factor when $\theta < {q+1 \over q+2}.$
In the next section we consider the case where loss is not only convex but also smooth.

\subsection{Finite Sample Bounds for  Smooth Loss Functions}
In this section,
we additionally assume that $V(y,\cdot)$ is differentiable and $V'(y,\cdot)$ is Lipschitz continuous with constant $L>0$, i.e., for any $y\in Y$ and $a,b\in \mR,$
$$ |V'(y,b) - V'(y,a)| \leq L|b-a|. $$
For the logistic loss in binary classification, see Example \ref{ExampleLoss}, it is easy to prove that both $V(y,\cdot)$ and $V'(y,\cdot)$ is Lipschitz continuous with constant $L = 1$, for all $y\in Y$.
With the above  smoothness assumption, we prove  the following convergence result.
\begin{thm}\label{MainRatesSmoothLoss}
Assume (\ref{EqCond1}) with $q \geq 0$, (\ref{varianceexpect}) with  $\tau\in[0,1]$, (\ref{decayapprox}) with $\beta \in (0,1]$ and (\ref{capacityB}) with $\zeta\in(0,2)$.
Assume that $V(y,\cdot)$ is differentiable and $V'(y,\cdot)$ is Lipschitz continuous with constant $L>0$.
Let $\eta_t=\eta_1 t^{-\theta}$ with $0\leq \theta <1$ and $0<\eta_1 \leq \min({1 - \theta \over 2|V|_0} , (L\kappa^2)^{-1}).$
If $T$ is the integer part of $\lceil m^{\gamma} \rceil$, then for any $0< \delta <1$, with confidence $1-\delta$, we have
$$
{\mathcal E}(f_T) - {\mathcal E}(f_\rho^V)
\leq
\widetilde{C} m^{-\alpha} \log \frac{2}{\delta},
$$
where the power indices $\gamma$ and $\alpha$ are defined as
\bea
\gamma &=&
\frac{2}{1-\theta} \frac{1}{\left(1 + 2 \beta\right) (2-\tau+\zeta \tau/2) + q (1 + \zeta/2)}, \\
\alpha &=&
\frac{\beta}{\beta (2 -\tau + \zeta \tau/2) + \left\{\frac{2 -\tau + \zeta \tau/2}{2} + \frac{q(1+ \zeta/2)}{2}\right\}},
\eea
and $\widetilde{C}$ is a constant independent of $m$ or $\delta$.
\end{thm}
The proof of this result will be given in Section \ref{sectionSampleError}. We can simplify the result by considering Lipschitz loss function ($q=0$)
and  setting $\tau=0$.
\begin{corollary}
Under the assumptions of Theorem \ref{MainRatesSmoothLoss}, let $q=0.$
If $T$ is the integer part of $\lceil m^{\gamma} \rceil$, then for any $0< \delta <1$, with confidence $1-\delta$, we have
$$
{\mathcal E}(f_T) - {\mathcal E}(f_\rho^V)
\leq
\widetilde{C} m^{-\alpha} \log \frac{2}{\delta},
$$
where the power indices $\gamma$ and $\alpha$ are defined as
\bea
\gamma =
{ 2 \over (1 - \theta) (2\beta + 1) (2 - \tau + \zeta \tau/2)}, \quad
\alpha = {2\beta \over (2\beta + 1) (2 - \tau + \zeta \tau/2)},
\eea
and $\widetilde{C}$ is a constant independent of $m$ or $\delta$.
\end{corollary}
The finite sample bound obtained  above is  essentially the same as the best possible bound obtained for general convex loss.
However, the important difference is that for smooth loss function,  a constant stepsize can be chosen and allows to  considerably improve the stopping rule.
Indeed, if for example we can consider the limit $\beta\to1$, $\tau \to0$,  we have that the stopping is  $O(m^{1/3})$, rather than $O(m^{2/3})$,  whereas the corresponding finite sample bounds is again $O(m^{-1/ 3})$.

\subsection{Iterative Regularization for Classification: Surrogate Loss Functions and  Hinge Loss}

We briefly discuss how the above results allows us to derive error bounds in binary classification.
In this latter case $Y =\{1, -1\}$ and a natural choice for the loss function is the misclassification loss given by
\begin{equation}\label{mloss}
V(y, b(x)) = \Theta\left( - y b(x)\right)
\end{equation}
for $b:X\to Y$ and $\Theta(a)=1$, if $a \geq 0$, and $\Theta(a)=0$ otherwise.
The corresponding generalization error, usually  denoted by ${\mathcal R}$, is  called mislassification risk, since it can be  shown to be
the probability of the event $\{(x, y) \in Z: y\not= b(x)\}$. The minimizer of the misclassification error is the Bayes rule $b_\rho: X \to Y$ given by
$$ b_\rho (x) = \left\{\begin{array}{ll}
1, & \hbox{if the conditional probability} \ \rho(y=1|x) \geq 1/2, \\
-1, & \hbox{otherwise.} \end{array}\right.
$$
 The misclassification loss~\eqref{mloss} is  neither convex nor smooth and thus leads to intractable problems. Moreover, the search of a solution among binary valued functions is also unfeasible.  In practice, a convex (so called {\em surrogate}) loss function is typically considered and a classifier is obtained by estimating a real function $f$
and then taking  its  sign defined as
$$ \hbox{sign}(f)(x) = \left\{\begin{array}{ll}
1, & \hbox{if} \ f (x) \geq 0, \\
-1, & \hbox{otherwise.} \end{array}\right.
$$
The question arises of if, and how, error bounds on the excess risk ${\mathcal E}(f)-{\mathcal E}(f_\rho^V)$ yields results on
 ${\mathcal R}( \text{sign} f)-{\mathcal R}(b_\rho)$. Indeed, so called {\em comparison} results are known relating these different error measures, see e.g.  \cite{CuckerZhou,SteinwartChristmann2008a}  and references therein. We discuss in particular the case of the hinge loss function, see Example~\ref{ExampleLoss},
 since in this case  for all measurable functions $f$ it holds that
$${\mathcal R}( \text{sign} f)-{\mathcal R}(b_\rho)\le
{\mathcal E}(f)-{\mathcal E}(f_\rho^V).
$$
Indeed, the hinge loss satisfies Assumption~\eqref{EqCond1} with $q=0$ and,  under Tsybakov noise condition, Assumption~\eqref{varianceexpect}.
Misclassification error bound, for the iterative regularization induced by the hinge loss, can then be obtained
as a corollary of Theorem \ref{MainRates} and using the above facts. Below we provide a simplified result.

%To illustrate the above stopping rule and probabilistic upper bounds, we consider a classification setting with $Y =\{1, -1\}$ and choose the hinge loss defined by
%\begin{equation}
%V(y, f(x)) = \left(1- y f(x)\right)_+ := \max\{1- y f(x), 0\}. \label{hingeloss}
%\end{equation}
%We recall that a classifier  is a binary valued function $b$ from $X$ to $Y$ and its misclassification error ${\mathcal R}(b)$ is defined as the probability of the event $\{(x, y) \in Z: y\not= b(x)\}$ of making wrong predictions. The performance of a classification algorithm can be measured by the excess misclassification error ${\mathcal R}(b) - {\mathcal R}(b_\rho)$.
%Given the output $f_T$ of Algorithm \ref{Alg} a corresponding classifier is defined by $\hbox{sign}(f_T)$ defined as
%%In this paper the classifier is given by the sign $\hbox{sign}(f_T)$ of the output function $f_T$ from (\ref{Alg}):
%$$ \hbox{sign}(f_T)(x) = \left\{\begin{array}{ll}
%1, & \hbox{if} \ f_T (x) \geq 0, \\
%-1, & \hbox{otherwise.} \end{array}\right.
%$$
%The following convergence results will be proved in Section \ref{sectionSampleError} as corollaries of Theorem \ref{MainRates}, when the approximation error decays quickly.

\begin{thm}\label{hingeRates1}
Let $Y =\{1, -1\}$ and $V$ be the hinge loss. Let $0<
\epsilon < \frac{1}{3}$ and (\ref{decayapprox}) is satisfied with $\beta \in(0,1]$. Let $\eta_t=\eta_1 t^{-\theta}$
with $\theta > 1/2$ and $0< \eta_1\leq
\min\left\{\frac{\sqrt{(1-\theta)}}{\sqrt{2}(\kappa +1)},
\frac{1-\theta}{4}\right\}.$ If (\ref{capacityB}) is valid with $\zeta\in(0,2)$
and $T$ is the integer part of
$\lceil m^{\frac{1}{(1-\theta)(2\beta+1)}} \rceil$, then with confidence
$1-\delta$, we have
\be \label{HingeLossBound}{\mathcal R}\left(\hbox{sign}(f_T)\right) - {\mathcal R}(f_c)
\leq \widetilde{C} {m}^{-\frac{\beta}{2\beta+1}} \log
\frac{2}{\delta}.
\ee
In particular, if $\beta>\frac{1 -  3\epsilon}{1 + 6\epsilon} $ with
$\epsilon\in(0,1/3),$ then with confidence
$1-\delta,$
$${\mathcal R}\left(\hbox{sign}(f_T)\right) - {\mathcal R}(f_c)
\leq \widetilde{C} {m}^{\epsilon-\frac{1}{3}} \log \frac{2}{\delta}.
$$
\end{thm}
The proof of the above result is given in Section \ref{sectionSampleError}, whereas we comment on the obtained rates in the next section. We add one of observation first. We  note that, as  illustrated by the the next result,  a different regularization strategy than early stopping can be considered, where the stopping rule is kept fixed while the step size is chosen in a distribution dependent way.
\begin{thm}\label{hingeRates}
Let $Y =\{1, -1\}$ and $V$ be the hinge loss given. Let $0< \epsilon < \frac{1}{3}$ and (\ref{decayapprox}) is satisfied with $1> \beta > \frac{4- 3\epsilon}{4+ 6\epsilon} $. Let $\eta_t=\eta_1 t^{-\theta}$ with $\theta = \frac{4 \beta -1 + 3 \epsilon (2 \beta +1)}{(2 \beta +1)(2 + 3 \epsilon)}$ and
$0< \eta_1\leq \min\left\{\frac{\sqrt{2(1-\theta)}}{\kappa +1}, \frac{1-\theta}{4}\right\}.$ If (\ref{capacityB}) is valid with $\zeta\in(0,2)$ and
$T$ is the integer part of $\lceil m^{\frac{2}{3} + \epsilon} \rceil$, then with confidence $1-\delta$, we have
$${\mathcal R}\left(\hbox{sign}(f_T)\right) - {\mathcal R}(f_c)
\leq \widetilde{C} m^{\frac{\epsilon}{4}- \frac{1}{3}} \log \frac{2}{\delta}.
$$
\end{thm}

\subsection{Comparison with Other Learning Algorithms}\label{sec:errcomp}
As mentioned before  iterative regularization has clear advantages from a computational point of view. The algorithm
reduces to a simple first order method with typically low iteration cost and allows to seemly compute the estimators corresponding
to different regularization level (the regularization path), a crucial fact since model selection needs to be performed.
It is natural to compare the obtained statistical  bounds with those for other  learning algorithms.
For general convex loss functions, the methods for which sharp bounds are available, are penalized empirical risk
minimization (Tikhonov regularization), i.e.
$$ f_{\bf z, \lambda} = \argmin_{f \in \mcHK } \left\{ \mcE_{\bf z}(f) + \lambda \|f\|_K^2 \right\}, \quad \lambda >0,$$
see e.g.  \cite{CuckerZhou,SteinwartChristmann2008a}  and references therein.
The best error bounds for Tikhonov regularization with Lipschitz loss functions, see e.g. \cite[Chapter 7]{SteinwartChristmann2008a},
are of order $O(m^{-\alpha'})$ with
$$
\alpha'= \min\left\{
\frac{2\beta}{\beta+1},
\frac{\beta}{(2-\zeta/2-\tau +\tau\zeta/2) \beta+\zeta/2}
\right\},
$$
which reduces to
$$
\alpha'=\frac{\beta}{\beta+1}
%\alpha'= \min\left\{
%\frac{2\beta}{\beta+1},
%\frac{\beta}{(2-\zeta/2) \beta+\zeta/2}, \quad
%\right\}
$$
if no variance assumption is made ($\tau=0$) and
% to $\alpha'=\frac{\beta}{\beta+1}$
  in capacity independent limit ($\zeta \to 2$).
While from Theorem~\ref{MainRates} for Lipschitz loss functions,    we see that the bound we obtain are of order $O(m^{-\alpha})$ with exponent
$$
\alpha=
\frac{2\beta}{(2 \beta + 1) (2-\tau+\zeta \tau/2)},
$$
reducing to $$\alpha=\frac{\beta}{2\beta+1}$$ in no variance and capacity independent limit.
The obtained bounds are  worse than the best ones available for Tikhnov regularization. However, the analysis of the latter does not take into account the optimization error and it is still an open question whether the best rate is preserved when  such an error is incorporated.
At this point we are prone to believe this gap to be a byproduct of
our analysis rather than a fundamental fact, and addressing this point should be a subject of further work.
Moreover, we note that our analysis allows to derive error bound for all  Nemitski loss functions.

Beyond Tikhnov regularization, we can compare with the  online regularization scheme for the hinge loss.
The online learning algorithms with regularization sequence $\{\lambda_t>0\}_t$ defined by
\be
f_{t+1}= \left\{\begin{array}{ll} (1 -
\eta_t \lambda_t) f_t, & \hbox{ if } y_t f_t(x_t) >1, \\
(1- \eta_t \lambda_t) f_t + \eta_t y_t K_{x_t}, & \hbox{ if } y_t
f_t(x_t) \leq 1. \end{array}\right. \label{onlinelambda}
\ee
were studied in \cite{YZ06, YeZhou}.  Our results  improves the results in \cite{YZ06, YeZhou} in two aspects. The bound obtained
in \cite{YZ06} is of the form $O(T^{\epsilon - \frac{1}{4}})$ while the bound in Theorem \ref{hingeRates} is of type $O(T^{\frac{9}{8} \epsilon - \frac{1}{2}})$ by substituting the expression $m^{\frac{2}{3} + \epsilon}$ for $T$. Moreover, our results are with high probability and promptly yields almost sure convergence whereas the results in \cite{YZ06} are only in expectation. We note that, interestingly,  sharp bounds for Lipschitz loss functions are derived in \cite{Pistol}, although the obtained results do not take into account  capacity and  variance assumptions that could lead to large improvements.

We next compare with the previous  results on iterative regularization.
%for convex smooth loss functions. In this latter case,
The only results available thus far have been obtained for the square loss, for which bounds have been first derived for gradient descent  in \cite{Buhlmann03}, but only for  a fixed design regression setting, and in \cite{YRC07}  for a general statistical learning setting. While the  bounds in \cite{YRC07} are suboptimal, they have later  been improved in \cite{bauer,capyao10,Raskutti14}.  Interestingly, sharp error bounds have also been proved for iterative regularization induced by other, potentially faster,  iterative techniques,  including incremental gradient \cite{RTV14},  conjugate gradient \cite{Blanchard10} and  the so called $\nu$-method \cite{bauer,capyao10}, an accelerated gradient descent technique related to Chebyshev method \cite{Engl1996}. The best obtained bounds are of order $O(m^{-\frac{2\beta}{2\beta+\zeta}})$ and can be shown to be optimal since they match a  corresponding minimax lower bound \cite{capdev07}. Holding not only for the square loss, but for general Nemitski loss functions, the bound obtained in Theorem~\ref{MainRatesSmoothLoss} is of order $O(m^{-\frac{2\beta}{(2+\zeta)(\beta+1)}})$, which is worse. In the capacity independent limit, the best available bound we obtain is of order $O(m^{-\frac{\beta}{2(\beta+1)}})$, whereas the  optimal bound is of order  $O(m^{-\frac{\beta}{\beta+1}})$. Also, in this case, the reason for the gap appears to be of technical reason and should be further studied.

Finally, before giving the proof of our results in  details, in the next section, we discuss the general error
decomposition underlying our approach, which highlights the interplay between statistics and optimization and could be also useful in other contexts.

\subsection{Error Decomposition}\label{sec:errdec}

Theorems \ref{MainRates} and \ref{MainRatesSmoothLoss} rely on a key error decomposition, that we derive next.
 The goal is to estimate the excess risk  ${\mathcal E}(f_T) - {\mathcal E}(f_\rho^V)$,  and the starting point is to split the error
by introducing a \emph{reference} function $f_*\in \mathcal{H}_K$,
\be
{\mathcal E}(f_T) - {\mathcal E}(f_\rho^V)={\mathcal E}(f_T) -{\mathcal E}(f_*) + {\mathcal E}(f_*)- {\mathcal E}(f_\rho^V).
 \ee
The above equation can be further developed by considering
\be
{\mathcal E}(f_T) - {\mathcal E}(f_\rho^V)=
%&=&
\left({\mathcal E}_{\bf z} (f_T) -{\mathcal E}_{\bf z} (f_*)\right) +
\left({\mathcal E}(f_T) - {\mathcal E}_{\bf z} (f_T)
 +{\mathcal E}_{\bf z} (f_*)  -  {\mathcal E} (f_*)\right)
  + \left({\mathcal E} (f_*) - {\mathcal E}(f_\rho^V)\right),
 \ee
Inspection of the above expression provides several insights. The first term is a computational error related to  optimization. It quantifies  the discrepancy between the empirical errors  of the iterate  defined by the  subgradient method and that of the reference function.  The last two terms are related to statistics. The second term is a sample error and  can be studied using empirical process theory, provided that a bound on the norm of the iterates (and of the reference function) is available.  Indeed, to get a sharper concentration the \emph{recentered} quantity
$$
\left\{\left({\mathcal E}(f_T) - {\mathcal E}(f_{\rho}^V) \right) - \left({\mathcal E}_{\bf z} (f_T) - {\mathcal E}_{\bf z}(f_{\rho}^V)\right)\right\}+
\left({\mathcal E}_{\bf z} (f_*) -{\mathcal E}_{\bf z} (f_\rho)\right) - \left({\mathcal E} (f_*) - {\mathcal E}(f_\rho^V)\right)
$$
can be considered \cite{CuckerZhou,SteinwartChristmann2008a}.  Note that the second addend can be negative so that we effectively only need to control
\be\label{FRsampleerror}
{\mathcal F}_{\bf z} (f_*) = \max\left\{({\mathcal E}_{\bf z}(f_*) - {\mathcal E}_{\bf z}(f_{\rho}^V)) - ({\mathcal E}(f_*) - {\mathcal E}(f_{\rho}^V)),\ 0\right\}.
\ee
Finally the last term suggests that  a natural choice for the reference function  is an {\em almost minimizer} of
the expected risk, having bounded norm,  and for which the approximation level can be quantified.  While there is a certain degree of freedom in the latter choice, in the following we will consider $f_*=f_\lambda$, the minimizer of~\eqref{targetf}.
With this latter choice we can control
$$
{\mathcal A}(f_*)= \left({\mathcal E} (f_*) - {\mathcal E}(f_\rho^V)\right)
$$
by $\mathcal{D}(\lambda)$ given in Assumption~\ref{ApproDef}. Indeed, other choices are possible, for example
$$
f_R=\argmin_{f\in B_R} {\mathcal E} (f).
$$
With this choice,  ${\mathcal A}(f_R)$ can be seen to be  another standard way to measure  approximation
properties  \cite{CuckerZhou,SteinwartChristmann2008a}.\\
Collecting some of the above observations, we have the following Lemma.
\begin{lemma}\label{Eq1}
  For $R>0,$ we have
  \be\begin{split}
{\mathcal E}(f_T) - {\mathcal E}(f_\rho^V) \leq \left\{\left({\mathcal E}(f_T) - {\mathcal E}(f_{\rho}^V) \right) - \left({\mathcal E}_{\bf z} (f_T) - {\mathcal E}_{\bf z}(f_{\rho}^V)\right) + {\mathcal F}_{\bf z} (f_*) \right\}
 + \left({\mathcal E}_{\bf z} (f_T) -{\mathcal E}_{\bf z} (f_*)\right) + {\mathcal A}(f_*).
\end{split}
\ee
\end{lemma}
In the next sections we proceed estimating the various error terms in the error. We will first deal with the computational error, the analysis of which is the main technical contribution of the paper and then proceed to consider the sample and approximation error terms. The best stopping
criterion and corresponding rates are derived  by suitably balancing the different error terms.

\section{Computational Error}\label{ProofResultConvexLoss}
In this section, we will bound the iterates and estimate the computational error, see Lemma~\ref{Eq1}.
\subsection{Bounds on Iterates}
We introduce the following key lemma, which will be used several times in our analysis.

\begin{lemma}\label{Lem1} For any fixed $f\in\mcHK$ and $t=1, \ldots, T$,
  \be\label{generalIter} \|f_{t+1}-f\|_{K}^2\leq
\|f_{t}-f\|_{K}^2+ \eta_t^2G_t^2
  +2\eta_t[\mcE_{\mathbf{z}}(f)-\mcE_{\mathbf{z}}(f_t)], \ee
  where \be\label{EqGt} G_t^2=\left\|\frac{1}{m}\sum_{j=1}^{m}V'_- (y_j, f_t(x_j)) K_{x_j}\right\|_K^2
  \leq c_q^2(\kappa +1)^{2q+2}\max\left\{1,\|f_t\|_K^{2q}\right\}.\ee
\end{lemma}
\begin{proof}
Computing inner product $\la f_{t+1}-f, f_{t+1}-f\ra_K$ with $f_{t+1}$ given by (\ref{Alg}) yields
$$ \|f_{t+1}-f\|_{K}^2=\|f_{t}-f\|_{K}^2+
  \eta_t^2 G_t^2
  +\frac{2\eta_{t}}{m}\sum_{j=1}^{m}V'_- (y_j, f_t(x_j)) \left\la K_{x_j}, f-f_{t}\right\ra_K.$$
Using the reproducing property
\be\label{reproProp}
f(x) = \la f, K_x\ra_K, \qquad \forall f\in {\mathcal H}_K, x\in X,
\ee
we get
\be\label{reprobound}
\|f\|_\infty \leq \kappa \|f\|_K, \qquad \forall f\in {\mathcal H}_K,
\ee
and
  \be\label{Eq13} \|f_{t+1}-f\|_{K}^2=\|f_{t}-f\|_{K}^2+
  \eta_t^2G_t^2
  +\frac{2\eta_{t}}{m}\sum_{j=1}^{m}V'_- (y_j, f_t(x_j))(f(x_j)-f_{t}(x_j)). \ee
Since $V (y_j, \cdot)$ is a convex function, we have
$$V'_- (y_j, a) (b-a)\leq V (y_j, b)-V (y_j, a), \qquad \forall a, b\in\mR.$$
Using this expression to (\ref{Eq13}) gives
  $$ \|f_{t+1}-f\|_{K}^2\leq \|f_{t}-f\|_{K}^2+
  \eta_t^2G_t^2
  + \frac{2\eta_{t}}{m}\sum_{j=1}^{m}\left[V(y_j, f(x_j))- V(y_j, f_{t}(x_j))\right],$$
  where the last term is exactly $2\eta_t[\mcE_{\mathbf{z}}(f)-\mcE_{\mathbf{z}}(f_t)]$.

By (\ref{EqCond1}), (\ref{reprobound}), and the observation $\|K_{x_j}\|_K = \sqrt{K(x_j, x_j)} \leq \kappa$, we find
\begin{eqnarray*}
G_t&=& \left\|\frac{1}{m}\sum_{j=1}^{m}V'_- (y_j, f_t(x_j))K_{x_j}\right\|_K
\leq \frac{\kappa}{m}\sum_{j=1}^{m} \left|V'_- (y_j, f_t(x_j))\right| \\
&\leq& \frac{\kappa}{m}\sum_{j=1}^{m} c_q (1+|f_t(x_j)|^q) \leq \kappa c_q (1+ \kappa^q \|f_t\|_K^q),
\end{eqnarray*}
and the desired bound follows.
\end{proof}

Using the above lemma, we can bound the iterates as follows.

\begin{lemma}\label{Lem4}
Let $0< \theta <1$ satisfy $\theta \geq \frac{q}{q+1}$ and $\eta_t=\eta_1 t^{-\theta}$ with $\eta_1$ satisfying (\ref{restr1}).
Then for $t=1, \ldots, T$,
 \be\label{Eq20}\|f_{t+1}\|_K\leq
t^{\frac{1-\theta}{2}}.\ee
\end{lemma}

\begin{proof}
We prove our statement by induction. Taking $f=0$ in Lemma \ref{Lem1}, we know that
$$
\|f_{t+1}\|_{K}^2 \leq
\|f_{t}\|_{K}^2+ \eta_t^2 G_t^2
  +2\eta_t[\mcE_{\mathbf{z}}(0)-\mcE_{\mathbf{z}}(f_t)] \leq \|f_{t}\|_{K}^2+ \eta_t^2 G_t^2
  +2\eta_t |V|_0. $$
This verifies (\ref{Eq20}) for the case $t=1$ since $f_1 =0$ and $\eta_1^2 c_q^2(\kappa +1)^{2q+2} +2\eta_1 |V|_0 \leq 1$.\\
Assume $\|f_t\|_K \leq (t-1)^{\frac{1-\theta}{2}}$ with $t \geq 2$. Then
$$ G_t^2 \leq c_q^2(\kappa +1)^{2q+2} (t-1)^{(1-\theta)q}. $$
Hence
\begin{eqnarray*}
\|f_{t+1}\|_{K}^2 &\leq& (t-1)^{1-\theta} + \eta_1^2 t^{-2 \theta} c_q^2(\kappa +1)^{2q+2} t^{(1-\theta)q} + 2 \eta_1 t^{-\theta} |V|_0 \\
&\leq& t^{1-\theta} \left\{\left(1-\frac{1}{t}\right)^{1-\theta} + \frac{\eta_1^2 c_q^2(\kappa +1)^{2q+2}}{t^{(q+1)\theta +1-q}} + \frac{2 \eta_1 |V|_0}{t}\right\}.
\end{eqnarray*}
Since $\left(1-\frac{1}{t}\right)^{1-\theta} \leq 1- \frac{1-\theta}{t}$ and the condition $\theta \geq \frac{q}{q+1}$ implies $(q+1)\theta +1-q \geq 1$, we have
$$ \|f_{t+1}\|_{K}^2 \leq t^{1-\theta} \left\{1- \frac{1-\theta}{t} + \frac{\eta_1^2 c_q^2(\kappa +1)^{2q+2}}{t} + \frac{2 \eta_1 |V|_0}{t}\right\}. $$
Finally we use the restriction (\ref{restr1}) for $\eta_1$ and find $\|f_{t+1}\|_{K}^2 \leq t^{1-\theta}$. This completes the induction procedure and proves our conclusion.
\end{proof}

By taking $f=f_t$ in (\ref{generalIter}), we see the following estimate for the difference $f_{t+1} -f_t$ from Lemmas \ref{Lem1} and \ref{Lem4}.

\begin{corollary}\label{normest} Under the assumption of Lemma \ref{Lem4}, we have for $t=1, \ldots, T$,
\begin{equation} \|f_{t+1} -f_t\|_K \leq  \eta_1 c_q (\kappa +1)^{q+1} t^{\frac{(1-\theta)q}{2} -\theta}. \label{ftdiffnorm}
\end{equation}
\end{corollary}

Observe from the restriction $\theta \geq \frac{q}{q+1}$ in Lemma \ref{Lem4} that the power index in (\ref{ftdiffnorm}) satisfies $\frac{(1-\theta)q}{2} -\theta \leq -\frac{q}{2(q+1)} \leq 0$.

\subsection{Computational Error for the Last Iterate}\label{subsectionEED}
In this subsection, we will estimate the computational error $\mcE_{\mathbf{z}}(f_T)-\mcE_{\mathbf{z}}(f_{*})$ for some $f_* \in {\mathcal H}_K.$
Some ideas for estimating the average error in our proof are from \cite{Boyd03,SZ13}.

\begin{lemma}\label{ExcessEmp}
Assume (\ref{EqCond1}) with $q \geq 0$. Let $f_{*} \in {\mathcal H}_K$. If $\eta_t=\eta_1 t^{-\theta}$ with $0< \theta <1$ satisfying $\theta > \frac{q}{q+1}$ and $\eta_1$ satisfying
(\ref{restr1}), then we have
\begin{eqnarray}
&&\mcE_{\mathbf{z}}(f_T) -\mcE_{\mathbf{z}}(f_{*}) \leq \left(\frac{\|f_{*}\|_{K}^2}{2 \eta_1} + c'_\theta \mcE_{\mathbf{z}}(f_{*}) + \widetilde{C}_1\right) \Lambda_T \nonumber \\
&& \ +\frac{T^{\theta}}{2\eta_1} \sum_{k=1}^{T-1} \frac{1}{k+1} \left[\frac{1}{k}\sum_{t=T-k+1}^{T} 2 \eta_{t} - 2 \eta_{T-k}\right] \left\{\mcE_{\mathbf{z}}(f_{T-k}) - \mcE_{\mathbf{z}}(f_{*})\right\}, \label{empexcessB}
\end{eqnarray}
where $\Lambda_T$ is  defined by
\begin{equation}\label{LamnbdaDef}
\Lambda_T = \left\{\begin{array}{ll}
T^{-(1-\theta)}, & \hbox{when} \ \theta > \frac{q+1}{q+2}, \\
(\log T) T^{-(1-\theta)}, & \hbox{when} \ \theta = \frac{q+1}{q+2}, \\
(\log T) T^{-(\theta(1+q) -q)}, & \hbox{when} \ \theta < \frac{q+1}{q+2}, \end{array}\right.
\end{equation}
$c'_\theta := \frac{1}{1-\theta} \left(1 + \max\{(2 + \log 4) (1 + \log t)t^{-\theta}: t\in\mN\}\right)$ and $\widetilde{C}_1$ is a constant depending on $q, \kappa, \theta$ (independent of $T, m$ or $f_{*}$ and given explicitly in the proof).
\end{lemma}

\begin{proof}
Lemma \ref{Lem1} is key in our proof. In particular, we shall apply the following equivalent form of inequality (\ref{generalIter}) from Lemma \ref{Lem1} several times with various choices of $f\in {\mathcal H}_K$:
\be
2 \eta_{t}\left[\mcE_{\mathbf{z}}(f_t) - \mcE_{\mathbf{z}}(f)\right] \leq \left\{\|f_{t}-f\|_{K}^2 - \|f_{t+1}-f\|_{K}^2\right\} + \eta_t^2 G_t^2. \label{generalIterEq}
\ee

{\sl Step 1: Error decomposition.} Decompose the weighted empirical error $2 \eta_{T} \mcE_{\mathbf{z}}(f_T)$ as
\bea
2 \eta_{T} \mcE_{\mathbf{z}}(f_T) &=& \frac{1}{2} \left\{2 \eta_{T}\mcE_{\mathbf{z}}(f_T) + 2 \eta_{T-1}\mcE_{\mathbf{z}}(f_{T-1})\right\} \\
 && + \frac{1}{2} 2 \eta_{T}\left\{\mcE_{\mathbf{z}}(f_T) - \mcE_{\mathbf{z}}(f_{T-1})\right\}  + \frac{1}{2} \left\{2 \eta_{T}- 2 \eta_{T-1}\right\}\mcE_{\mathbf{z}}(f_{T-1})\\ &=& \frac{1}{3} \left\{2 \eta_{T}\mcE_{\mathbf{z}}(f_T) + 2 \eta_{T-1}\mcE_{\mathbf{z}}(f_{T-1}) + 2 \eta_{T-2}\mcE_{\mathbf{z}}(f_{T-2})\right\} \\
&& + \frac{1}{2 \times 3} \left\{2 \eta_{T}\left[\mcE_{\mathbf{z}}(f_T) - \mcE_{\mathbf{z}}(f_{T-2})\right] + 2 \eta_{T-1}\left[\mcE_{\mathbf{z}}(f_{T-1}) - \mcE_{\mathbf{z}}(f_{T-2})\right]\right\} \\
&& + \frac{1}{2} 2 \eta_{T}\left\{\mcE_{\mathbf{z}}(f_T) - \mcE_{\mathbf{z}}(f_{T-1})\right\}  + \frac{1}{2} \left\{2 \eta_{T}- 2 \eta_{T-1}\right\}\mcE_{\mathbf{z}}(f_{T-1}) \\
&& + \frac{1}{2 \times 3} \left\{\left[2 \eta_{T} - 2 \eta_{T-2}\right] + \left[2 \eta_{T-1} - 2 \eta_{T-2}\right]\right\} \mcE_{\mathbf{z}}(f_{T-2}).
\eea
Repeating the above process by means of the decomposition
\bea &&\frac{1}{k} \sum_{j=0}^{k-1} 2 \eta_{T-j}\mcE_{\mathbf{z}}(f_{T-j}) = \frac{1}{k+1} \sum_{j=0}^{k} 2 \eta_{T-j}\mcE_{\mathbf{z}}(f_{T-j}) \\
&& + \frac{1}{k(k+1)} \sum_{j=0}^{k-1} 2 \eta_{T-j}\left\{\mcE_{\mathbf{z}}(f_{T-j}) -\mcE_{\mathbf{z}}(f_{T-k})\right\} + \frac{1}{k(k+1)} \sum_{j=0}^{k-1} \left\{2 \eta_{T-j} -2 \eta_{T-k}\right\}\mcE_{\mathbf{z}}(f_{T-k}) \eea
with $k=3, \ldots, T-1$, we know that
\bea 2 \eta_{T}\mcE_{\mathbf{z}}(f_T) &=& \frac{1}{T} \sum_{j=0}^{T-1}  2 \eta_{T-j}\mcE_{\mathbf{z}}(f_{T-j}) + \sum_{k=1}^{T-1} \frac{1}{k(k+1)} \sum_{j=0}^{k-1} 2 \eta_{T-j}\left\{\mcE_{\mathbf{z}}(f_{T-j}) -\mcE_{\mathbf{z}}(f_{T-k})\right\} \\
&& + \sum_{k=1}^{T-1}\frac{1}{k(k+1)} \sum_{j=0}^{k-1} \left\{2 \eta_{T-j} -2 \eta_{T-k}\right\}\mcE_{\mathbf{z}}(f_{T-k}).
\eea
Hence the following error decomposition holds true:
\begin{eqnarray}
&& 2 \eta_{T} \left\{\mcE_{\mathbf{z}}(f_T) -\mcE_{\mathbf{z}}(f_{*})\right\}= \frac{1}{T} \sum_{t=1}^{T} 2 \eta_{t} \left\{\mcE_{\mathbf{z}}(f_{t}) -\mcE_{\mathbf{z}}(f_{*})\right\}  \nonumber \\
&& + \sum_{k=1}^{T-1} \frac{1}{k(k+1)} \sum_{t=T-k+1}^{T} 2 \eta_{t}\left\{\mcE_{\mathbf{z}}(f_{t}) -\mcE_{\mathbf{z}}(f_{T-k})\right\}  \nonumber \\
&& + \left\{\frac{1}{T} \sum_{t=1}^{T} 2 \eta_{t} -2 \eta_{T} + \sum_{k=1}^{T-1} \frac{1}{k+1} \left[\frac{1}{k}\sum_{t=T-k+1}^{T} 2 \eta_{t} - 2 \eta_{T-k}\right]\right\}\mcE_{\mathbf{z}}(f_{*})  \nonumber\\
&& +
\sum_{k=1}^{T-1} \frac{1}{k+1} \left[\frac{1}{k}\sum_{t=T-k+1}^{T} 2 \eta_{t} - 2 \eta_{T-k}\right] \left\{\mcE_{\mathbf{z}}(f_{T-k}) - \mcE_{\mathbf{z}}(f_{*})\right\}. \label{errorDecom}
\end{eqnarray}

{\sl Step 2: Average error in the first term of (\ref{errorDecom}).} Choosing $f=f_{*}$ in (\ref{generalIterEq}) and taking summation over $t=1, \ldots, T$ together with (\ref{EqGt}) and Lemma \ref{Lem4} yields
\bea
\sum_{t=1}^{T} 2 \eta_{t}\left\{\mcE_{\mathbf{z}}(f_{t}) -\mcE_{\mathbf{z}}(f_{*})\right\} &\leq&  \|f_{1}-f_{*}\|_{K}^2 - \|f_{T+1}-f_{*}\|_{K}^2  + \sum_{t=1}^T \eta_t^2 G_t^2 \\
&\leq& \|f_{*}\|_{K}^2 + \sum_{t=1}^T \eta_1^2 c_q^2(\kappa +1)^{2q+2} t^{q(1-\theta) -2 \theta}.
\eea
Since $1>\theta > \frac{q}{q+1}$, we find $-2 < q(1-\theta) -2 \theta <0$. Moreover, $q(1-\theta) -2 \theta < -1$ if and only if $\theta > \frac{q+1}{q+2}$. The following bound for the first term of (\ref{errorDecom}) then follows
\bea && \frac{1}{T}\sum_{t=1}^{T} 2 \eta_{t}\left\{\mcE_{\mathbf{z}}(f_{t}) -\mcE_{\mathbf{z}}(f_{*})\right\} \\
&& \leq \left\{\begin{array}{ll}
\left(\|f_{*}\|_{K}^2 +C_{q, \kappa} \frac{(2+q) \theta -q}{(2+q) \theta -q -1}\right) T^{-1}, & \hbox{when} \ \theta > \frac{q+1}{q+2}, \\
\left(\|f_{*}\|_{K}^2 +2C_{q, \kappa}\right) (\log T) T^{-1}, & \hbox{when} \ \theta = \frac{q+1}{q+2}, \\
\left(\|f_{*}\|_{K}^2 +C_{q, \kappa} \frac{2}{q+1 -(2+q) \theta}\right) T^{q -(2+q) \theta}, & \hbox{when} \ \theta < \frac{q+1}{q+2}, \end{array}\right. \eea
where $C_{q, \kappa}$ is the constant given by
$$ C_{q, \kappa} = \eta_1^2 c_q^2(\kappa +1)^{2q+2}. $$

{\sl Step 3: Moving average error in the second term of (\ref{errorDecom}).} Let $k\in \{1, \ldots, T-1\}$. Choosing $f=f_{T-k}$ in (\ref{generalIterEq}) and taking summation over $t=T-k + 1, \ldots, T$ yields
\bea
\sum_{t=T-k+1}^{T} 2 \eta_{t}\left\{\mcE_{\mathbf{z}}(f_{t}) -\mcE_{\mathbf{z}}(f_{T-k})\right\} \leq  \|f_{T-k +1}-f_{T-k}\|_{K}^2 + \sum_{t=T-k+1}^{T} \eta_t^2 G_t^2
\eea
By Corollary \ref{normest},
$$ \|f_{T-k +1}-f_{T-k}\|_{K}^2 \leq  \eta_1^2 c_q^2 (\kappa +1)^{2(q+1)} (T-k)^{(1-\theta)q -2\theta}. $$
This bound is the term with $t=T-k+1$ of the following estimate which is a consequence of Lemma \ref{Lem4}
$$ \sum_{t=T-k+1}^{T} \eta_t^2 G_t^2 \leq \sum_{t=T-k+1}^{T} \eta_1^2 c_q^2(\kappa +1)^{2q+2} t^{q(1-\theta) -2 \theta}. $$
Hence
$$ \sum_{t=T-k+1}^{T} 2 \eta_{t}\left\{\mcE_{\mathbf{z}}(f_{t}) -\mcE_{\mathbf{z}}(f_{T-k})\right\} \leq C_{q, \kappa} \left[\sum_{t=T-k+1}^{T} t^{q(1-\theta) -2 \theta} + (T-k)^{q(1-\theta) -2\theta}\right]. $$
Denote $q^{*} =2 \theta - q(1-\theta)$. We know that $0< q^{*} <2$ and $q^{*} =1$ when $\theta = \frac{q+1}{q+2}$. So
$$ \sum_{t=T-k+1}^{T} t^{q(1-\theta) -2 \theta} \leq \int_{T-k}^T x^{-q^{*}} d x \leq  \left\{\begin{array}{ll}
\frac{T^{1-q^{*}} - (T-k)^{1-q^{*}}}{1-q^{*}}, & \hbox{when} \ \theta \not= \frac{q+1}{q+2}, \\
\log \frac{T}{T-k}, & \hbox{when} \ \theta = \frac{q+1}{q+2}. \end{array}\right. $$

When $\theta < \frac{q+1}{q+2}$, we have $q^{*} <1$ and for $k \leq \frac{T}{2}$, we see from the mean value theorem that
$$ \frac{T^{1-q^{*}} - (T-k)^{1-q^{*}}}{1-q^{*}} = T^{1-q^{*}} \frac{1 - (1-\frac{k}{T})^{1-q^{*}}}{1-q^{*}} \leq  T^{1-q^{*}} \frac{(1-q^{*}) (1-\frac{k}{T})^{-q^{*}}\frac{k}{T}}{1-q^{*}}$$
which is exactly $(T-k)^{-q^*}k$ and bounded by $2^{q^{*}} T^{-q^{*}}k$. It follows that
\bea
&& \sum_{k=1}^{T-1} \frac{1}{k(k+1)} \sum_{t=T-k+1}^{T} 2 \eta_{t}\left\{\mcE_{\mathbf{z}}(f_{t}) -\mcE_{\mathbf{z}}(f_{T-k})\right\} \\
&& \leq 2 C_{q, \kappa} \sum_{k \leq T/2} \frac{1}{k(k+1)} 2^{q^{*}} T^{-q^{*}}k + 2 C_{q, \kappa} \sum_{T-1 \geq k > T/2} \frac{1}{k(k+1)}\frac{T^{1-q^{*}}}{1-q^{*}} \\
&& \leq 2 C_{q, \kappa} \left(2^{q^{*}} + \frac{2}{1-q^{*}}\right) (\log T) T^{q(1-\theta) -2 \theta}.
\eea

When $\theta = \frac{q+1}{q+2}$, we we see from the mean value theorem that
$$ \log \frac{T}{T-k} = -\log \left(1-\frac{k}{T}\right) \leq \frac{k}{T} \frac{1}{1-\frac{k}{T}} =\frac{k}{T-k}. $$
It follows that
\bea
&& \sum_{k=1}^{T-1} \frac{1}{k(k+1)} \sum_{t=T-k+1}^{T} 2 \eta_{t}\left\{\mcE_{\mathbf{z}}(f_{t}) -\mcE_{\mathbf{z}}(f_{T-k})\right\} \\
&& \leq C_{q, \kappa} \sum_{k=1}^{T-1} \frac{1}{(T-k) k} = C_{q, \kappa} \frac{1}{T} \sum_{k=1}^{T-1} \left\{\frac{1}{k} + \frac{1}{T-k}\right\} \\
&& \leq 4 C_{q, \kappa} \frac{\log T}{T}.
\eea

When $\theta > \frac{q+1}{q+2}$, we have $q^{*} >1$ and for $k \leq \frac{T}{2}$,
$$ \frac{T^{1-q^{*}} - (T-k)^{1-q^{*}}}{1-q^{*}} = T^{1-q^{*}} \frac{(1-\frac{k}{T})^{1-q^{*}}-1}{q^{*}-1} \leq 2^{q^{*}} T^{-q^{*}}k. $$
Then
\bea
&& \sum_{k=1}^{T-1} \frac{1}{k(k+1)} \sum_{t=T-k+1}^{T} 2 \eta_{t}\left\{\mcE_{\mathbf{z}}(f_{t}) -\mcE_{\mathbf{z}}(f_{T-k})\right\} \\
&& \leq 2^{q^{*} +1}  C_{q, \kappa}T^{-q^{*}} \sum_{k=1}^{T-1} \frac{1}{k+1}  \leq 2^{q^{*} +1} C_{q, \kappa} T^{-q^{*}} \log T  \\
&& \leq 2^{q^{*} +1} C_{q, \kappa} \frac{1}{q^{*} -1} T^{-1}.
\eea
Thus the second term of (\ref{errorDecom}) can also be bounded as
\bea && \sum_{k=1}^{T-1} \frac{1}{k(k+1)} \sum_{t=T-k+1}^{T} 2 \eta_{t}\left\{\mcE_{\mathbf{z}}(f_{t}) -\mcE_{\mathbf{z}}(f_{T-k})\right\} \\
&& \leq \left\{\begin{array}{ll}
\frac{2^{q^{*} +1} C_{q, \kappa}}{q^{*} -1} T^{-1}, & \hbox{when} \ \theta > \frac{q+1}{q+2}, \\
4 C_{q, \kappa} (\log T) T^{-1}, & \hbox{when} \ \theta = \frac{q+1}{q+2}, \\
2 C_{q, \kappa} \left(2^{q^{*}} + \frac{2}{1-q^{*}}\right) (\log T) T^{q -(2+q) \theta}, & \hbox{when} \ \theta < \frac{q+1}{q+2}. \end{array}\right. \eea

{\sl Step 4: Error concerning $\mcE_{\mathbf{z}}(f_{*})$ in the third term of (\ref{errorDecom}).} Let $k \in \{1, \ldots, T\}$. We have
$$ \frac{1}{k}\sum_{t=T-k+1}^{T} 2 \eta_{t} \leq 2 \eta_1 \frac{1}{k}\sum_{t=T-k+1}^{T} \int_{t-1}^t x^{-\theta} d x \leq 2 \eta_1 \frac{T^{1-\theta} - (T-k)^{1-\theta}}{k(1-\theta)}. $$
Putting this estimate in the coefficient of the third term of (\ref{errorDecom}), we find
\bea
&& \frac{1}{T} \sum_{t=1}^{T} 2 \eta_{t} -2 \eta_{T} + \sum_{k=1}^{T-1} \frac{1}{k+1} \left[\frac{1}{k}\sum_{t=T-k+1}^{T} 2 \eta_{t} - 2 \eta_{T-k}\right] \\
&& \leq 2 \eta_1 \frac{T^{-\theta}}{1-\theta} - 2 \eta_1 T^{-\theta} +  2 \eta_1 \sum_{k=1}^{T-1} \frac{1}{k+1} \left[\frac{T^{1-\theta} - (T-k)^{1-\theta}}{k(1-\theta)} - (T-k)^{-\theta}\right] \\
&& =  \frac{2 \eta_1 \theta}{1-\theta} T^{-\theta} +  \frac{2 \eta_1}{1-\theta}\sum_{k=1}^{T-1} \frac{1}{k(k+1)} \left[T^{1-\theta} - (T-k)^{1-\theta} - k(1-\theta)(T-k)^{-\theta}\right] \\
&& = \frac{2 \eta_1 \theta}{1-\theta} T^{-\theta} -  \frac{2 \eta_1}{1-\theta}\sum_{k=1}^{T-1} \frac{T^{1-\theta}}{k(k+1)} g(\frac{k}{T}),
\eea
where $g: [0, 1) \to \mR$ is the function defined by
$$ g(u) = -1 + (1-u)^{1-\theta} + (1-\theta)u (1-u)^{-\theta}, \qquad u\in [0, 1). $$
A simple computation gives its derivative
$$ g'(u) =\theta (1-\theta)u (1-u)^{-1-\theta}.  $$
So $g$ is an increasing function and is positive on $(0, 1)$ by noting $g(0) =0$. Observe that
$$ \frac{T}{k(k+1)} = \frac{T}{k} - \frac{T}{k+1} = \int^{(k+1)/T}_{k/T} u^{-2} d u =  \int^{k/T}_{(k-1)/T} (u + \frac{1}{T})^{-2} d u $$
and $g(\frac{k}{T}) \geq g(u)$ for $u \in ((k-1)/T, k/T)$. Hence
\bea \sum_{k=1}^{T-1} \frac{T}{k(k+1)} g(\frac{k}{T}) &\geq& \sum_{k=1}^{T-1} \int^{k/T}_{(k-1)/T} (u + \frac{1}{T})^{-2} g(u) d u = \int^{(T-1)/T}_{0} (u + \frac{1}{T})^{-2} g(u) d u \\
&=& \left[-(u + \frac{1}{T})^{-1} g(u)\right]^{(T-1)/T}_{0} + \int^{(T-1)/T}_{0} (u + \frac{1}{T})^{-1} g' (u) d u \\
&=& -g(\frac{T-1}{T}) + \theta (1-\theta) \int^{(T-1)/T}_{0} \frac{u (1-u)^{-1-\theta}}{u + \frac{1}{T}} d u.
\eea
By the definition of the function $g$, we see
$$-g(\frac{T-1}{T})= 1 - T^{\theta -1} -(1-\theta) \left(T^{\theta} - T^{\theta-1}\right)= 1 - \theta T^{\theta -1} -(1-\theta) T^{\theta}. $$
Writing
$$\frac{u}{u + \frac{1}{T}} = 1- \frac{1}{T} \frac{1}{u + \frac{1}{T}} = 1- \frac{1}{T} \frac{1}{1 + \frac{1}{T} -(1-u)} = 1- \frac{1}{T+1} \left(1- \frac{1-u}{1 + \frac{1}{T}}\right)^{-1}, $$
we use the Taylor expansion for the integral and find
\bea \int^{(T-1)/T}_{0} \frac{u (1-u)^{-1-\theta}}{u + \frac{1}{T}} d u &=& \int^{(T-1)/T}_{0} (1-u)^{-1-\theta} \left\{1- \frac{1}{T+1} \sum_{k=0}^\infty \left(\frac{1-u}{1 + \frac{1}{T}}\right)^k \right\} d u \\
&=& \frac{T^\theta -1}{\theta} - \frac{1}{T+1} \sum_{k=0}^\infty \left(\frac{T}{T+1}\right)^k \frac{1-T^{\theta -k}}{k -\theta} \\
&\geq&  \frac{T}{T+1}\frac{T^\theta -1}{\theta} - \frac{1}{T+1} \sum_{k=1}^\infty \left(\frac{T}{T+1}\right)^k \frac{1}{k -\theta}.
\eea
We notice that $\left(1 + \frac{1}{T}\right)^T \geq 2$ for any $T \geq 2$ (with limit $e$), which implies
$$ \left(\frac{T}{T+1}\right)^k = \frac{1}{\left(1 + \frac{1}{T}\right)^k} \leq 2^{1-\ell}, \qquad \forall (\ell -1) T +1 \leq k \leq \ell T, \ \ell\in\mN. $$
It follows that
 \begin{eqnarray}\label{Eq21} && \sum_{k=1}^\infty \left(\frac{T}{T+1}\right)^k \frac{1}{k -\theta} \leq \sum_{\ell =1}^\infty \sum_{k= (\ell -1) T +1}^{\ell T} \frac{2^{1-\ell}}{k -\theta}\nonumber \\
 && \leq \frac{1}{1-\theta} + \log \frac{T-\theta}{1-\theta} + \sum_{\ell =2}^\infty 2^{1-\ell} \log \frac{\ell T-\theta}{(\ell -1)T-\theta} \nonumber \\
 && \leq \frac{1}{1-\theta} + \log \frac{1}{1-\theta} + 2 \log 4 + \log T.  \end{eqnarray}
Therefore, we have
\bea \sum_{k=1}^{T-1} \frac{T}{k(k+1)} g(\frac{k}{T}) &\geq& 1 - \theta T^{\theta -1} -(1-\theta) T^{\theta} + (1-\theta) \frac{T}{T+1}\left(T^\theta -1\right) \\
&& - \theta (1-\theta)\left(\frac{1}{1-\theta} + \log \frac{1}{1-\theta} + 2 \log 4\right) \frac{1 + \log T}{T+1} \\
&& \geq \theta + \frac{1-\theta}{T+1} - T^{\theta -1} - \left(2 + \log 4\right) \frac{1 + \log T}{T+1}.
\eea
This tells us that the third term of (\ref{errorDecom}) can be estimated as
\bea
&& \left\{\frac{1}{T} \sum_{t=1}^{T} 2 \eta_{t} -2 \eta_{T} + \sum_{k=1}^{T-1} \frac{1}{k+1} \left[\frac{1}{k}\sum_{t=T-k+1}^{T} 2 \eta_{t} - 2 \eta_{T-k}\right]\right\}\mcE_{\mathbf{z}}(f_{*}) \\
&& \leq \frac{2 \eta_1}{1-\theta} T^{-\theta} \left(T^{\theta -1} + \left(2 + \log 4\right) \frac{1 + \log T}{T+1}\right) \mcE_{\mathbf{z}}(f_{*}) \leq 2 \eta_1 T^{-\theta} c'_\theta T^{\theta -1} \mcE_{\mathbf{z}}(f_{*}).
\eea

Putting all the above estimates for the first three terms into (\ref{errorDecom}), we see that the desired bound (\ref{empexcessB}) holds true with the constant $\widetilde{C}_1$ given explicitly by
$$  \widetilde{C}_1 = \left\{\begin{array}{ll}
\eta_1 c_q^2(\kappa +1)^{2q+2} \frac{(2+q) \theta -q + 2^{(2+q) \theta -q}}{(2+q) \theta -q -1}, & \hbox{when} \ \theta > \frac{q+1}{q+2}, \\
6 \eta_1 c_q^2(\kappa +1)^{2q+2}, & \hbox{when} \ \theta = \frac{q+1}{q+2}, \\
\eta_1 c_q^2(\kappa +1)^{2q+2} \left(2^{(2+q) \theta -q} + \frac{3}{q+1 -(2+q) \theta}\right), & \hbox{when} \ \theta < \frac{q+1}{q+2}. \end{array}\right. $$
The proof of Lemma \ref{ExcessEmp} is complete.
\end{proof}

Lemma \ref{ExcessEmp} is useful and can be used in a stochastic  convex optimization problem, other than learning. In what follows, we shall see that how it can be used in our specified learning problems. For notational simplicity, with $\widetilde{R} >0$ we denote
\begin{equation}\label{emperror}
{\mathcal M}_{\bf z} (\widetilde{R}) = \sup_{f\in B_{\widetilde{R}}} \max\left\{{\mathcal E}_{\bf z}(f_{\rho}^V) - {\mathcal E}_{\bf z}(f),\ 0\right\}.
\end{equation}
\begin{pro}
  \label{lemmaConvexOptimi2}
  Under the assumptions of Lemma \ref{ExcessEmp}, we have
\begin{eqnarray}
 {\mathcal E}_{\bf z}(f_T) - {\mathcal E}_{\bf z} (f_*) &\leq& \frac{3}{1-\theta} \mathcal{M}_{\mathbf z}\left(T^{\frac{1-\theta}{2}}\right)
+ \left(c'_\theta \Lambda_T  + {3 \over 1 - \theta}\right) \left({\mathcal F}_{\bf z} (f_*) + {\mathcal A}(f_*)\right) \nonumber \\
 && + \frac{\|f_*\|_K^2}{2 \eta_1} \Lambda_T + \widetilde{C}_2 \Lambda_T, \label{generalB} \end{eqnarray}
where $\widetilde{C}_2$ is the constant given by $\widetilde{C}_2=c'_\theta \left(|V|_0 + c_q (1 + \|f_\rho^V\|_\infty^q) \|f_\rho^V\|_\infty\right) + \widetilde{C}_1.$
\end{pro}

\begin{proof}
Note that by Lemma \ref{ExcessEmp}, we have (\ref{empexcessB}).
The first term in the bound (\ref{empexcessB}) involves the empirical error ${\mathcal E}_{\bf z} (f_*)$ which can be estimated as
\bea {\mathcal E}_{\bf z} (f_*) &=& ({\mathcal E}_{\bf z} (f_*) - {\mathcal E}_{\bf z} (f_{\rho})) - ({\mathcal E} (f_*) - {\mathcal E} (f_{\rho})) + ({\mathcal E} (f_*) - {\mathcal E} (f_{\rho})) + {\mathcal E}_{\bf z}(f_{\rho}^V)\\
 &&\leq {\mathcal F}_{\bf z} (f_{*}) + {\mathcal A}(f_*) + {\mathcal E}_{\bf z}(f_{\rho}^V). \eea
Also, condition (\ref{EqCond1}) implies
\begin{equation}\label{Vdecay}
|V(y, f_\rho^V(x))| \leq |V|_0 + c_q (1 + |f_\rho^V(x)|^q) |f_\rho^V(x)| \leq |V|_0 + c_q (1 + \|f_\rho^V\|_\infty^q) \|f_\rho^V\|_\infty.\nonumber
\end{equation}
Hence,
\be\label{frhoSampleerror}
{\mathcal E}_{\bf z}(f_{\rho}^V) \leq |V|_0 + c_q (1 + \|f_\rho^V\|_\infty^q) \|f_\rho^V\|_\infty.\nonumber
\ee
With these, we can bound the first term of (\ref{empexcessB}) as
\bea
\left(\frac{\|f_*\|_{K}^2}{2 \eta_1} + c'_\theta \mcE_{\mathbf{z}}(f_*) + \widetilde{C}_1\right) \Lambda_T \leq c'_\theta \left({\mathcal F}_{\bf z} (f_*) + {\mathcal A}(f_*)\right) \Lambda_T + \left(\frac{\|f_*\|_{K}^2}{2 \eta_1} + \widetilde{C}_2 \right) \Lambda_T.
\eea

What is remained is to estimate the second term of (\ref{empexcessB}) denoted as
$$ J_{T, {\bf z}} := \frac{T^{\theta}}{2\eta_1} \sum_{k=1}^{T-1} \frac{1}{k+1} \left[2 \eta_{T-k} - \frac{1}{k}\sum_{t=T-k+1}^{T} 2 \eta_{t}\right] \left\{\mcE_{\mathbf{z}} (f_*) - \mcE_{\mathbf{z}}(f_{T-k})\right\}. $$
Denote $\widetilde{R} = T^{\frac{1-\theta}{2}}$. Lemma \ref{Lem4} tells us that $f_{k} \in B_{\widetilde{R}}$ for each $k=1,\cdots,T$. It follows that for $k=1,\cdots,T-1,$
\bea
\mcE_{\mathbf{z}} (f_*) - \mcE_{\mathbf{z}}(f_{T-k}) &=& \left\{\left(\mcE_{\mathbf{z}} (f_*) - {\mathcal E}_{\bf z}(f_{\rho}^V)\right) - \left({\mathcal E} (f_*) - {\mathcal E}(f_{\rho}^V)\right)\right\} \\
&& + \left({\mathcal E} (f_*) - {\mathcal E}(f_{\rho}^V)\right) + {\mathcal E}_{\bf z}(f_{\rho}^V) - \mcE_{\mathbf{z}}(f_{T-k}) \\
&\leq&  {\mathcal F}_{\bf z} (f_*) + \mathcal{A}(f_*) + \mathcal{M}_{\mathbf z}(\widetilde{R}).
\eea
By the choice of the step sizes, $2 \eta_{T-k} - \frac{1}{k}\sum_{t=T-k+1}^{T} 2 \eta_{t} \geq 0$ for any $k \in\{1, \ldots, T-1\}$.
Therefore, $J_{T, {\bf z}}$ can be bounded by
\bea
J_{T, {\bf z}} \leq \frac{T^{\theta}}{2\eta_1} \sum_{k=1}^{T-1} \frac{1}{k+1} \left[2 \eta_{T-k} -\frac{1}{k}\sum_{t=T-k+1}^{T} 2 \eta_{t}\right] \left\{{\mathcal F}_{\bf z} (f_*) + \mathcal{A}(f_*) + \mathcal{M}_{\mathbf z}(\widetilde{R})\right\}.
\eea
Now we need to bound the above summation. Note that, for each $k$,
$$2 \eta_{T-k} -\frac{1}{k}\sum_{t=T-k+1}^{T} 2 \eta_{t} = \frac{2 \eta_{1}}{k}\sum_{t=T-k+1}^{T} \left((T-k)^{-\theta} - t^{-\theta}\right). $$
Applying the mean value theorem to the function $g(x) = -x^{-\theta}$ on $[T-k, t]$ with $t\in \{T-k+1, \ldots, T\}$, we find that for some $c\in (T-k, t)$,
$$(T-k)^{-\theta} - t^{-\theta} = g(t) - g(T-k) = (t-(T-k)) g' (c) \leq (t-(T-k)) \theta (T-k)^{-\theta -1}. $$
Hence
\bea
&& \qquad \sum_{k=1}^{T-1} \frac{1}{k+1} \left[2 \eta_{T-k} -\frac{1}{k}\sum_{t=T-k+1}^{T} 2 \eta_{t}\right] \\
&&\leq 2 \eta_1 \theta \sum_{k< T/2} \frac{(T-k)^{-\theta-1}}{k(k+1)} \sum_{t=T-k+1}^{T}(t-T+k) + \sum_{k \geq T/2} \frac{1}{k+1} 2 \eta_{T-k}\\
&&\leq 2 \eta_1 \theta \sum_{k< T/2} \frac{(T-k)^{-\theta-1}}{k(k+1)} \frac{k(k+1)}{2}+  \sum_{k \geq T/2} \frac{2}{T} 2 \eta_{T-k} \\
&&\leq \eta_1 \theta \sum_{k< T/2} (T-k)^{-\theta-1} + \frac{4 \eta_1}{T} \sum_{k \geq T/2} (T-k)^{-\theta}
\leq \frac{6\eta_1}{1-\theta} T^{-\theta}.
\eea
Thus
$$J_{T, {\bf z}} \leq \frac{3}{1-\theta} \left\{{\mathcal F}_{\bf z}(f_*) + \mathcal{A}(f_*) + \mathcal{M}_{\mathbf z}(\widetilde{R})\right\}.$$
Then the desired bound  follows from Lemma \ref{ExcessEmp}.
\end{proof}

\subsection{Computational Errors for Weighted Average and Best Iterate}

\begin{lemma}\label{empexcessAverge}
  Under the assumptions of Lemma \ref{ExcessEmp}, let $g_T = a_T$ (or $g_T = b_T$). Then
  $$ \mcE_{\mathbf{z}}(b_{T}) -\mcE_{\mathbf{z}}(f_{*}) \leq \left( {2 \|f_*\|_K^2 \over \eta_1} + \overline{C}_1 \right) \overline{\Lambda}_T,$$
  where $\overline{\Lambda}_T$ is given by
  \be\label{OverlinLamnbdaDef}
  \Lambda_T = \left\{\begin{array}{ll}
T^{-(1-\theta)}, & \hbox{when} \ \theta > \frac{q+1}{q+2}, \\
(\log T) T^{-(1-\theta)}, & \hbox{when} \ \theta = \frac{q+1}{q+2}, \\
T^{-(\theta(1+q) -q)}, & \hbox{when} \ \theta < \frac{q+1}{q+2}, \end{array}\right.
  \ee and $\overline{C}_1$ is a constant depending on $q,\kappa,\theta$ (independent of $T,m$ or $f_*$ and given explicitly in the proof.)
\end{lemma}
Note that there is a subtle difference between $\overline{\Lambda}_T$ and $\Lambda_T$ defined by (\ref{LamnbdaDef}),
where the later term has an extra $\log m$ for $\theta < {q+1 \over q+2}.$
\begin{proof}
For any $u \in \mR,$ we have
$$
\sum_{t=1}^T \eta_t ({\mathcal E}_{\bf z}(f_t)-u) \ge \left(\sum_{t=1}^T \eta_t\right)\min_{t=i, \dots, T}{\mathcal E}_{\bf z}(f_t) - \left(\sum_{t=1}^T \eta_t\right) u.
$$
and by convexity of ${\mathcal E}_{\bf z}$,
$$
{\mathcal E}_{\bf z} (a_T)= {\mathcal E}_{\bf z} \left(   \sum_{t=1}^T \omega_t  f_t \right)
\le   \sum_{t=1}^T \omega_t  {\mathcal E}_{\bf z}(f_T)= \frac{1}{\sum_{t=1}^T\eta_t}\sum_{t=1}^T \eta_t {\mathcal E}_{\bf z}(f_t).
$$
Therefor, we have
$$
 {\mathcal E}_{\bf z} (b_T)-u \le \frac {1}{ \sum_{t=1}^T \eta_t} \sum_{t=1}^T \eta_t ({\mathcal E}_{\bf z}(f_t)-u)
$$
and
$$
 {\mathcal E}_{\bf z}(a_T)-u   \le \frac {1}{ \sum_{t=1}^T \eta_t}  \sum_{t=1}^T \eta_t ({\mathcal E}_{\bf z}(f_t) -  u).
$$
We thus get
\be\label{Eqconvex}
 {\mathcal E}_{\bf z}(g_T) - {\mathcal E}_{\bf z}(f^*) \le \frac {1}{ \sum_{t=1}^T \eta_t}  \sum_{t=1}^T \eta_t ({\mathcal E}_{\bf z}(f_t) -  {\mathcal E}_{\bf z}(f^*)).
\ee
  Following Step 2 of the proof of Lemma \ref{ExcessEmp}, we have
\bea && \sum_{t=1}^{T} 2 \eta_{t}\left\{\mcE_{\mathbf{z}}(f_{t}) -\mcE_{\mathbf{z}}(f_{*})\right\} \\
&& \leq \left\{\begin{array}{ll}
\left(\|f_{*}\|_{K}^2 +C_{q, \kappa} \frac{(2+q) \theta -q}{(2+q) \theta -q -1}\right) , & \hbox{when} \ \theta > \frac{q+1}{q+2}, \\
\left(\|f_{*}\|_{K}^2 +2C_{q, \kappa}\right) (\log T), & \hbox{when} \ \theta = \frac{q+1}{q+2}, \\
\left(\|f_{*}\|_{K}^2 +C_{q, \kappa} \frac{2}{q+1 -(2+q) \theta}\right) T^{(1+q) -(2+q) \theta}, & \hbox{when} \ \theta < \frac{q+1}{q+2}. \end{array}\right. \eea
Introducing the above inequality into (\ref{Eqconvex}), and using $\sum_{t=1}^{T} \eta_t \geq \eta_1 \int_{t=1}^{T+1} u^{-\theta} du \geq \eta_1 T^{1-\theta} / 2$, we get our desired result with $\overline{C}_1$ given by
\bea \overline{C}_1
 = \left\{\begin{array}{ll}
2 \eta_1 c_q^2(\kappa +1)^{2q+2} \frac{(2+q) \theta -q}{(2+q) \theta -q -1} , & \hbox{when} \ \theta > \frac{q+1}{q+2}, \\
4\eta_1 c_q^2(\kappa +1)^{2q+2}  , & \hbox{when} \ \theta = \frac{q+1}{q+2}, \\
2 \eta_1 c_q^2(\kappa +1)^{2q+2} \frac{2}{q+1 -(2+q) \theta}  , & \hbox{when} \ \theta < \frac{q+1}{q+2}. \end{array}\right. \eea
\end{proof}
While the above proof is shorter and easier than the proof of Lemma \ref{ExcessEmp}, it is surprising that the computational error bounds for the last iterate and the average (or the best one) are roughly of the same order.

\subsection{Iterate Bound and Computational Error for Smooth Loss Functions}
 The following result can be proved by using the fact that $V'(y,\cdot)$ is Lipschitz.

\begin{lemma}\label{ERMDifference}
  Let $0<\eta_t \leq (L\kappa^2)^{-1}$ for all $t \in \mN.$ Assume that $V(y,\cdot)$ is differentiable and $V'(y,\cdot)$ is Lipschitz continuous with constant $L>0$. Then we have
  $$ \mcE_{\mathbf{z}}(f_T) -\mcE_{\mathbf{z}}(f_{*}) \leq {\|f_*\|_K^2 \over \sum_{k=1}^T 2\eta_k }.$$
In particular, if $\eta_t = \eta_1 t^{-\theta}$ with $\theta \in [0,1)$ satisfying $\eta_1 \leq (L\kappa^2)^{-1},$ then
$$ \mcE_{\mathbf{z}}(f_T) -\mcE_{\mathbf{z}}(f_{*}) \leq {\|f_*\|_K^2 \over \eta_1 } T^{\theta - 1}.$$
\end{lemma}
\begin{proof}
  Since $V'(y,\cdot)$ is Lipschitz with constant $L$ for any $y \in Y,$ we have for any $a,b\in \mR,$
  $$ V(y,b) \leq V(y,a) + V'(y,a) (b-a) + {L \over 2} (b-a)^2.$$
  Choosing $y=y_j,$ $b= f_{t+1}(x_j)$ and $a = f_t(x_j),$ according to the reproducing kernel property
  (\ref{reproProp}) and (\ref{reprobound}), we get for $j=1,\cdots,m$ and $t\in \mN,$
  $$ V(y_j, f_{t+1}(x_j)) \leq V(y_j,f_t(x_j)) + V'(y_j,f_{t}(x_j)) \la f_{t+1} - f_{t}, K_{x_j} \ra_K + {L \kappa^2 \over 2} \|f_{t+1} - f_t\|_K^2.$$
  Summing up over $j=1,\cdots, m$, with $G_t = {1\over m} \sum_{j=1}^m V'(y_j,f_{t}(x_j)) K_{x_j}$, we get
  $$ \mcE_{\bf z} (f_{t+1}) \leq  \mcE_{\bf z} (f_{t}) + \la f_{t+1} - f_{t}, G_t \ra_K + {L \kappa^2 \over 2} \|f_{t+1} - f_t\|_K^2. $$
  Introducing with (\ref{Alg}), noting that $\eta_t < (L\kappa^2)^{-1},$ we get
  \be\label{SmoothOneIteration}
  \mcE_{\bf z} (f_{t+1}) \leq  \mcE_{\bf z} (f_{t}) - {\eta_t \over 2} \|G_t\|_K^2.
  \ee
  By the convexity of $V(y,\cdot)$, it is easy to prove that
  $$ \mcE_{\bf z} (f_{t}) \leq  \mcE_{\bf z} (f_{*}) + \la f_t - f_*, G_t\ra_K. $$
  Introducing this inequality into (\ref{SmoothOneIteration}), we get
  \bea
  \mcE_{\bf z} (f_{t+1}) &\leq& \mcE_{\bf z} (f_{*}) + {1 \over 2\eta_t}\left( 2\eta_t \la f_t - f_*, G_t\ra_K - {\eta_t^2} \|G_t\|_K^2 \right)\\
  &=& \mcE_{\bf z} (f_{*}) + {1 \over 2\eta_t}\left( \|f_t - f_*\|_K^2 - \|f_t - f_* - \eta_t G_t\|_K^2 \right)\\
  &=& \mcE_{\bf z} (f_{*}) + {1 \over 2\eta_t}\left( \|f_t - f_*\|_K^2 - \|f_{t+1} - f_* \|_K^2 \right),
  \eea
so that,
  \be\label{KeyInequSmooth} 2\eta_t(\mcE_{\bf z} (f_{t+1}) - \mcE_{\bf z} (f_{*})) \leq \|f_t - f_*\|_K^2 - \|f_{t+1} - f_* \|_K^2.\ee
  Summing up over $t=1 \cdots, T,$ with $f_1 = 0,$ we have
  $$ \sum_{t=1}^T 2\eta_t(\mcE_{\bf z} (f_{t+1}) - \mcE_{\bf z} (f_{*})) \leq \|f_1 - f_*\|_K^2 - \|f_{T+1} - f_* \|_K^2 \leq \|f_*\|_K^2.$$
  By (\ref{SmoothOneIteration}), we have $\mcE_{\bf z} (f_{T+1}) \leq \mcE_{\bf z} (f_{t+1})$ for all $t \leq T.$ It thus follows that
  $$  \sum_{t=1}^T 2\eta_t (\mcE_{\bf z} (f_{T+1}) - \mcE_{\bf z} (f_{*})) \leq \sum_{t=1}^T 2\eta_t(\mcE_{\bf z} (f_{t+1}) - \mcE_{\bf z} (f_{*})) \leq \|f_*\|_K^2, $$
  which leads to the first argument of the lemma. The rest of the proof can be finished by noting that
  $$ \sum_{t=1}^T \eta_t \geq \eta_1 \int_1^{T+1} u^{-\theta} du \geq \eta_1 {T^{1-\theta} \over 2}.$$
\end{proof}
Using the above lemma, we can bound the iterates as follows.
\begin{lemma}\label{BoundSmoothLoss}
  Under the assumptions of Lemma \ref{ERMDifference}, we have for $t=1, \cdots, T,$
  $$ \|f_{t+1}\|_K \leq \sqrt{2 |V|_0 \sum_{k=1}^t \eta_k} .$$
  In particular, if $\eta_t = \eta_1 t^{-\theta}$ with $\theta \in [0,1)$ satisfying $\eta_1 \leq {1 - \theta \over 2|V|_0},$
  then
  $$ \|f_{t+1}\|_K \leq t^{1 - \theta \over 2} .$$
\end{lemma}
\begin{proof}
  Choosing $f_* = 0$ in (\ref{KeyInequSmooth}), we get for $k =1, \cdots,t,$
  $$ \|f_{k+1}\|_K^2 \leq \|f_{k}\|_K^2 + 2 \eta_k (\mcE_{\bf z} (0) - \mcE_{\bf z} (f_{k+1})) \leq \|f_{k}\|_K^2 + 2 \eta_k |V|_0.$$
  Applying this relationship iteratively for $k=t,\cdots,1,$ with $f_1=0$, we get
  $$ \|f_{t+1}\|_K^2 \leq 2 |V|_0 \sum_{k=1}^t \eta_k , $$
  which leads to the first conclusion.
  The second inequality can be proved by noting that
  $$ \sum_{k=1}^t \eta_k = \eta_1 \sum_{k=1}^t k^{-\theta} \leq \eta_1 \left( 1 +  {t^{1-\theta} -1 \over 1 - \theta}\right) \leq \eta_1 {t^{1-\theta} \over 1 - \theta}. $$
\end{proof}

\section{Sample Error and Finite Sample Bounds}\label{sectionSampleError}
In this subsection, we will estimate sample errors and then prove our main results.
\subsection{Sample Errors}
We first bound the sample error $\mathcal{F}_{\bf z} (f_*)$ for some fixed $f_* \in {\mathcal H}_K$ as follows. This is done by applying  Bernstein inequality.
\begin{lemma} Assume condition (\ref{EqCond1}) and (\ref{varianceexpect}) hold. For any $f_* \in \mathcal{H}_K$ with $\|f_*\|_K \leq R$, where $R \geq 1$, with confidence at least $1 - {\delta \over 2},$
  \label{SampleErrorBound1}
  \be\label{fRB}
 \quad \mathcal{F}_{\bf z} (f_*) \leq (C'_1 + 2\sqrt{ c_\tau})\log \frac{2}{\delta} \max\left\{\frac{R^{q +1}}{m}, \ \left(\frac{R^{2+q-\tau } } {m} \right)^{1 \over 2-\tau}, \ {\mathcal A}(f_*)\right\},
\ee
where $C'_1$ is a constant independent of $T,m,\delta$, given explicitly in the proof.
\end{lemma}
\begin{proof}
   We apply  Bernstein inequality which asserts that, for a random variable $\xi$ bounded by $\widetilde{M} >0$ and for any $\epsilon>0$,
$$ \hbox{Prob} \biggl\{{1\over m}\sum_{i=1}^m \xi (z_i) - \mE (\xi) > \epsilon \biggr\} \le
\exp \biggl\{-\frac{m \epsilon^2}{2\bigl(\sigma^2 (\xi) +
\frac{1}{3} \widetilde{M} \epsilon\bigr)}\biggr\}. $$
Here the random variable $\xi$ on $Z$ is given by $\xi (x, y) = V(y, f_*) -V(y, f_\rho^V(x))$.
The increment condition (\ref{EqCond1}) implies that $\xi$ is bounded by $M:=C_1' R^{q +1},$ where $C'_1$ is the constant given by
$$
C_1':= c_q \left(\kappa+\kappa^{q+1} + \|f_\rho^V\|_\infty + \|f_\rho^V\|_\infty^{q+1}\right) .
$$
By condition (\ref{varianceexpect}), its variance $\sigma^2 (\xi)$ is bounded by
$$c_\tau R^{2 + q-\tau}
\left\{{\mathcal E} (f_*)
-{\mathcal E} (f_\rho^V)\right\}^\tau \leq c_\tau R^{2 + q-\tau}
({\mathcal A}(f_*))^\tau. $$
Solving the quadratic equation from the Bernstein inequality, we see that with confidence at least $1-\frac{\delta}{2}$, there holds
\bea
\mathcal{F}_{\bf z} (f_*) &\leq& \frac{2 M \log \frac{2}{\delta}}{3 m}  +\sqrt{\frac{2\log \frac{2}{\delta}}{m} \sigma^2 (\xi)} \nonumber \\
&\leq&  (C'_1 + 2\sqrt{ c_\tau}) \log \frac{2}{\delta} \max\left\{\frac{R^{q +1}}{m}, \ \frac{R^{1 + \frac{q-\tau}{2}}\left({\mathcal A}(f_*)\right)^{\frac{\tau}{2}}}{\sqrt{m}}\right\}.
\eea
Applying an elementary inequality
\be\label{ElementaryInequ}
 x^{\tau} y^{1-\tau}\leq \tau x+(1-\tau)y, \qquad \forall \ \tau\in[0,1], \ x, y\geq 0
\ee yields
\bea \frac{R^{1 + \frac{q-\tau}{2}}\left({\mathcal A}(f_*)\right)^{\frac{\tau}{2}}}{\sqrt{m}} &=& \left[\left(\frac{R^{2+q-\tau } } {m} \right)^{1 \over 2-\tau} \right]^{1- {\tau \over 2}} \left({\mathcal A}(f_*)\right)^{\frac{\tau}{2}}\\
& \leq&  \left(1- {\tau \over 2}\right) \left(\frac{R^{2+q-\tau } } {m} \right)^{1 \over 2-\tau} + \frac{\tau}{2} {\mathcal A}(f_*).
\eea
Then the desired result follows.
\end{proof}
We next bound the empirical process over a ball $B_{\widetilde{R}}$ for some $\widetilde{R}>0.$ To do this, we need the following concentration inequality. Its proof is similar to that of Proposition 6 in \cite{WYZ07}, as well as applying \cite[Theorem 3.5]{Steinwart2009} and (\cite[Exercise 6.8]{SteinwartChristmann2008a}).
We omit the proof.
\begin{lemma}\label{concenIneq}
Let ${\mathcal G}$ be a set of measurable
functions on ${\mathcal Z}$, and $B, c>0, \tau \in [0, 1]$ be constants such
that each function $f\in {\mathcal G}$ satisfies $\|f\|_\infty \le B$
and $\mE (f^2)\le c (\mE f)^\tau$. If for some $a \ge B^{\zeta}$ and $\zeta\in (0, 2)$,
$$ \mE_{\bf z}[\log {\mathcal N} ({\cal G}, \epsilon,d_{2,{\bf z}})]\le
a \epsilon^{-\zeta},\qquad \forall \epsilon >0, $$ then there
exists a constant $c_\zeta'$ depending only on $\zeta$ such that for any
$b>0$, with probability at least $1-e^{-b}$, there holds
$$ \mE f- {1\over m}\sum_{i=1}^m
f(z_i)\le {1\over 2}\eta^{1-\tau}(\mE f)^\tau+ c_\zeta' \eta +
2\Big({c b \over m}\Big)^{1/(2-\tau)}+{18 B b \over m}, \qquad
\forall f\in {\cal G}, $$ where
$$\eta :=\max\bigg\{c^{2-\zeta \over
4-2\tau +\zeta\tau}\Big( \displaystyle{a\over m}\Big)^{2\over
4-2\tau+\zeta\tau}, \ B^{2-\zeta \over 2+\zeta}\Big(\displaystyle { a\over
m}\Big)^{2\over 2+\zeta}\bigg\}.$$
\end{lemma}
The following lemma is essentially contained in \cite{WYZ07}. We report  a short proof for the sake of completeness.

\begin{lemma}
  \label{SampleErrorBound2}
  Assume (\ref{EqCond1}) with $q \geq 0$, (\ref{varianceexpect}) with  $\tau\in[0,1]$, (\ref{decayapprox}) with $\beta \in (0,1]$ and (\ref{capacityB}) with $\zeta\in(0,2)$.
  Let $\widetilde{R} > 1.$ Then
  with confidence at least $1-\frac{\delta}{2}$, there holds for every $g\in B_{\widetilde{R}},$
\be\label{SampleErrorBound}\begin{split}
&\left({\mathcal E}(g) - {\mathcal E}(f_{\rho}^V)\right) - \left({\mathcal E}_{\bf z}(g) - {\mathcal E}_{\bf z}(f_{\rho}^V)\right)\\
\leq& {1\over 2} \left({\mathcal E}(g) - {\mathcal E}(f_\rho^V)\right) + C_3' \log \frac{2}{\delta}
\max\left\{\left(\frac{\widetilde{R}^{\frac{q(2 + \zeta) + (4-2\tau +\zeta \tau)}{2}}}{m}\right)^{\frac{2}{4-2\tau+\zeta \tau}}, \ {\widetilde{R}^{q+1} \over m^{\frac{2}{2+\zeta}}}, \
\Big({ \widetilde{R}^{2 + q-\tau} \over m}\Big)^{1 \over 2-\tau} \right\}, \end{split}
\ee
and
  \be\label{emperrorBound}
\mathcal{M}_{\bf z}(\widetilde{R}) \leq C_3' \log \frac{2}{\delta}
\max\left\{\left(\frac{\widetilde{R}^{\frac{q(2 + \zeta) + (4-2\tau +\zeta \tau)}{2}}}{m}\right)^{\frac{2}{4-2\tau+\zeta \tau}}, \ {\widetilde{R}^{q+1} \over m^{\frac{2}{2+\zeta}}}, \
\Big({ \widetilde{R}^{2 + q-\tau} \over m}\Big)^{1 \over 2-\tau} \right\}.
\ee
Here $C'_3$ is a constant independent of $T,m,\delta$, given explicitly in the proof.
\end{lemma}
\begin{proof}
We first apply Lemma \ref{concenIneq} to the function set
$$ {\mathcal G} = \left\{f(x, y) = V(y, g(x)) - V(y, f_\rho^V(x)): \ g\in B_{\widetilde{R}}\right\}. $$
Condition (\ref{varianceexpect}) tells us that with $c= c_\tau \widetilde{R}^{2 + q-\tau}$, each function $f \in {\mathcal G}$ satisfies
$\mE (f^2)\le c (\mE f)^\tau$. Also, condition (\ref{EqCond1}) implies that $\|f\|_\infty $ is bounded by
$ \widetilde{M} := C'_1 \widetilde{R}^{q +1}.$
Notice from $(\ref{EqCond1})$ that for $f,f'\in \mathcal{G},$
$$|f(x,y) - f'(x,y)| = |V(y,g(x)) - V(y,g'(x))| \leq c_q (1 + \kappa^q) \widetilde{R}^{q} |g(x) - g'(x)|,$$
there holds
$$\mathcal{N}(\mathcal{G},\epsilon,d_{2,{\bf z}}) \leq \mathcal{N}\left(B_{\widetilde R},{\epsilon \over c_q (1 + \kappa^q) \widetilde{R}^{q}},d_{2,{\bf z}}\right) \leq \mathcal{N}\left(B_{1},{\epsilon \over c_q (1 + \kappa^q) \widetilde{R}^{q+1}},d_{2,{\bf z}}\right).$$
Hence, condition (\ref{capacityB}) yields the covering number condition in Lemma \ref{concenIneq} with $a= c_\zeta c_q^\zeta (1 + \kappa^q)^\zeta \widetilde{R}^{(q+1)\zeta}$.
So we apply Lemma \ref{concenIneq} and find that with confidence at least $1-\frac{\delta}{2}$, there holds for every $f\in {\cal G}$,
$$ \mE (f)  - {1\over m}\sum_{i=1}^m f(z_i) \leq {1\over 2}\eta^{1-\tau}(\mE f)^\tau+ c_\zeta' \eta +
2\Big({c_\tau \widetilde{R}^{2 + q-\tau} \log \frac{2}{\delta} \over m}\Big)^{1 \over 2-\tau}+{18 \widetilde{M} \log \frac{2}{\delta} \over m}, $$
where
\bea
\eta &=&\max\left\{\left(c_\tau\widetilde{R}^{2 + q-\tau}\right)^{2- \zeta \over 4-2\tau +\zeta \tau} \left(\frac{c_\zeta c_q^\zeta (1 + \kappa^q)^\zeta \widetilde{R}^{(q+1)\zeta}}{m}\right)^{2 \over 4-2\tau+\zeta \tau},\right. \\
&& \left. \widetilde{M}^{2- \zeta \over 2+\zeta}\left(\frac{c_\zeta c_q^\zeta (1 + \kappa^q)^\zeta \widetilde{R}^{(q+1)\zeta}}{m}\right)^{2\over 2+\zeta}\right\} \\
&\leq& C'_2 \max\left\{\left(\frac{\widetilde{R}^{\frac{q(2 + \zeta) + (4-2\tau +\zeta \tau)}{2}}}{m}\right)^{\frac{2}{4-2\tau+\zeta \tau}}, \ {\widetilde{R}^{q+1} \over m^{\frac{2}{2+\zeta}}}\right\},
\eea
 where $C'_2$ is the constant given by
$$\left(c_\tau^{2- \zeta} c_\zeta^2 c_q^{2\zeta} (1 + \kappa^q)^{2\zeta} \right)^{1 \over 4-2\tau +\zeta \tau} + {C'_1}^{2- \zeta \over 2+\zeta}\left(c_\zeta c_q^\zeta (1 + \kappa^q)^\zeta \right)^{2\over 2+\zeta} .$$
Apply the elementary inequality (\ref{ElementaryInequ}) which yields $\eta^{1-\tau}(\mE f)^\tau \leq \eta + \mE f,$ and notice that $\mE (f) = {\mathcal E}(g) - {\mathcal E}(f_\rho^V)$ while ${1\over m}\sum_{i=1}^m
f(z_i) = {\mathcal E}_{\bf z}(g)- {\mathcal E}_{\bf z}(f_\rho^V)$. We get that with confidence at least $1-\frac{\delta}{2}$, there holds for every $g\in B_{\widetilde{R}},$ we have
\be\begin{split}
&\left({\mathcal E}(g) - {\mathcal E}(f_{\rho}^V)\right) - \left({\mathcal E}_{\bf z}(g) - {\mathcal E}_{\bf z}(f_{\rho}^V)\right)\\
\leq& \left({1\over 2} + c_\zeta'\right) \eta + {1\over 2} \left({\mathcal E}(g) - {\mathcal E}(f_\rho^V)\right)  +
2\Big({c_\tau \widetilde{R}^{2 + q-\tau} \over m}\Big)^{1 \over 2-\tau}\log \frac{2}{\delta} +{18 \widetilde{M} \log \frac{2}{\delta} \over m}, \end{split}
\nonumber,\ee
which leads to (\ref{SampleErrorBound}) with
$$C_3' = \left({1\over 2} + c_\zeta'\right) C_2' + 2{c_\tau }^{1 \over 2-\tau} + 18 \widetilde{M}.$$
Now, introducing (\ref{SampleErrorBound}) into the equality
\bea
{\mathcal E}_{\bf z}(f_{\rho}^V) - {\mathcal E}_{\bf z}(g) = \left\{\left({\mathcal E}(g) - {\mathcal E}(f_{\rho}^V)\right) - \left({\mathcal E}_{\bf z}(g) - {\mathcal E}_{\bf z}(f_{\rho}^V)\right)\right\} -\left({\mathcal E}(g) - {\mathcal E}(f_{\rho}^V)\right),
\eea
with ${\mathcal E}(g) - {\mathcal E}(f_\rho^V) \geq 0$ and by recalling the definition of $\mathcal{M}_{\bf z}(\widetilde{R})$, we can derive (\ref{emperrorBound}). The proof is completed.
\end{proof}

\subsection{Deriving the Finite Sample Bounds}
We have the following result, which will be used for the proof of Theorem \ref{MainRates}.
\begin{pro}\label{MainEstimates}
 Assume (\ref{EqCond1}) with $q \geq 0$, (\ref{varianceexpect}) with  $\tau\in[0,1]$, (\ref{decayapprox}) with $\beta \in (0,1]$ and (\ref{capacityB}) with $\zeta\in(0,2)$. Let $\eta_t=\eta_1 t^{-\theta}$ with $0< \theta <1$ satisfying $\theta > \frac{q}{q+1}$ and $\eta_1$ satisfying
(\ref{restr1}). Let $f_* \in \mathcal{H}_K$ be such that $\|f_*\|_K \leq R$, where $R \geq 1.$
If $1 \leq R \leq T^{\frac{1-\theta}{2}}$ and $T^{\frac{q (1-\theta)}{2}} m^{- \frac{2}{2+\zeta}}\leq 1,$ then with confidence $1-\delta$, we have
\be
{\mathcal E}(f_T) - {\mathcal E}(f_\rho^V)
\leq \widetilde{C}_3 \log \frac{2}{\delta} \max\left\{\left(\frac{T^{\frac{(1-\theta)(q(2 + \zeta) + (4-2\tau +\zeta \tau))}{4}}}{m}\right)^{\frac{2}{4-2\tau+\zeta \tau}}, \ R^{2} \Lambda_T,\ {\mathcal A}(f_*) \right\}.
\label{MainBound}
\ee
where $\widetilde{C}_3$ is a constant independent of $T, m, \delta$, given explicitly in the proof.
\end{pro}

\begin{proof}
Recall Lemma~\ref{Eq1}.
Let $\tilde{R} = T^{1 - \theta \over 2}.$ Introducing with (\ref{generalB}), we have
\begin{eqnarray}
&& {\mathcal E}(f_T) - {\mathcal E}(f_\rho^V) \leq \left\{\left({\mathcal E}(f_T) - {\mathcal E}(f_{\rho}^V) \right) - \left({\mathcal E}_{\bf z} (f_T) - {\mathcal E}_{\bf z}(f_{\rho}^V)\right)\right\} + \frac{3}{1-\theta} \mathcal{M}_{\mathbf z}\left(\tilde{R}\right) \nonumber \\
&& \qquad + \left(c'_\theta \Lambda_T  + {4 - \theta \over 1 - \theta}\right) \left({\mathcal F}_{\bf z} (f_*) + {\mathcal A}(f_*)\right)  + \frac{R^2}{2 \eta_1} \Lambda_T + \widetilde{C}_2 \Lambda_T. \nonumber \end{eqnarray}
Applying lemmas \ref{SampleErrorBound1} and \ref{SampleErrorBound2} with $g = f_T$, with $R\in [1,\widetilde{R}]$, we know that with confidence at least $1-\delta$,
\begin{eqnarray}
 {\mathcal E}(f_T) - {\mathcal E}(f_\rho^V)
&\leq& C_4'\log \frac{2}{\delta} \max\biggl\{\left(\frac{\widetilde{R}^{\frac{q(2 + \zeta) + (4-2\tau +\zeta \tau)}{2}}}{m}\right)^{\frac{2}{4-2\tau+\zeta \tau}}, \ {\widetilde{R}^{q+1}  \over m^{2 \over 2+\zeta}}, \nonumber \\
&& \quad  {\widetilde{R}^{2 + q-\tau \over 2-\tau} \over m^{1 \over 2-\tau}}, \  R^2 \Lambda_T,\ {\mathcal A}(f_*)\biggr\}
+ {1\over 2} \left({\mathcal E}(f_T) - {\mathcal E}(f_\rho^V)\right) ,\label{totalB}
\end{eqnarray}
where $C'_4$ is the constant given by
\begin{eqnarray*}
C_4' = \frac{4-\theta}{1-\theta} C_3' + \left( c'_\theta + \frac{4-\theta}{1-\theta} \right) \left(1 + C_1' + 2\sqrt{c_{\tau} }\right) + \frac{1}{2 \eta_1} + \widetilde{C}_2.
\end{eqnarray*}
Since $\widetilde{R}^{q} m^{-2 /(2+\zeta)}\leq 1$ and $\tau\in[0,1],\zeta\in(0,2)$ one finds
$$\left(\frac{\widetilde{R}^{\frac{q(2 + \zeta) + (4-2\tau +\zeta \tau)}{2}}}{m}\right)^{\frac{2}{4-2\tau+\zeta \tau}} \cdot \ {m^{2 \over 2+\zeta} \over \widetilde{R}^{q+1} } = \left\{ {\widetilde{R}^{q} \over  m^{2 \over 2+\zeta}}\right\}^{\frac{-(1-\tau)(2-\zeta)}{4-2\tau+\zeta \tau}}  \geq 1,$$
and $$ \left(\frac{\widetilde{R}^{\frac{q(2 + \zeta) + (4-2\tau +\zeta \tau)}{2}}}{m}\right)^{\frac{2}{4-2\tau+\zeta \tau}} \cdot  {m^{1 \over 2-\tau} \over \widetilde{R}^{2 + q-\tau \over 2-\tau} } = \left(\widetilde{R}^{2q (1-\tau)}m^{\tau}\right)^{\frac{\zeta}{(2-\tau)(4-2\tau+\zeta \tau)}} \geq 1.$$
Subtracting ${1\over 2} \left({\mathcal E}(f_T) - {\mathcal E}(f_\rho^V)\right)$ from both sides of (\ref{totalB}), and setting $\widetilde{C}_3 = 2 C'_4,$ we get the desired results.
\end{proof}

Now we are in a position to prove the explicit probabilistic upper bounds stated in Theorem \ref{MainRates}.

\proof[Proof of Theorem \ref{MainRates}]
We will use Proposition \ref{MainEstimates} with $f_* = f_{\lambda}$ to prove our result.
Define a power index $\tilde{\theta}$ as
\begin{equation}\label{tildeThetaDef}
\tilde{\theta} = \left\{\begin{array}{ll}
1-\theta, & \hbox{when} \ \theta \geq \frac{q+1}{q+2}, \\
\theta(1+q) -q, & \hbox{when} \ \theta < \frac{q+1}{q+2}, \end{array}\right.
\end{equation}
Comparing this with the definition (\ref{LamnbdaDef}) for $\Lambda_T$, we see that
$$ \Lambda_T = \left\{\begin{array}{ll}
T^{-\tilde{\theta}}, & \hbox{when} \ \theta > \frac{q+1}{q+2}, \\
T^{-\tilde{\theta}} \log T, & \hbox{when} \ \theta \leq \frac{q+1}{q+2}. \end{array}\right.
$$
From the definition of $\mathcal{D}(\lambda)$,
we have
\be \label{normflambda} \mathcal{A}(f_{\lambda}) \leq \mathcal{D}(\lambda) \qquad \mbox{and} \qquad \lambda \|f_{\lambda}\|_K^2 \leq \mathcal{D}(\lambda),\ee
which implies $\|f_{\lambda}\|_K \leq \sqrt{ \mathcal{D}(\lambda) / \lambda } = R.$
Balancing the orders of the last two terms of (\ref{MainBound}) by setting
\be
\lambda = \Lambda_T, \label{RDef}
\ee
we find that the last two terms of (\ref{MainBound}) can be bounded as
$$ \max\left\{R^2 \Lambda_T,\ {\mathcal A}(f_{\lambda})\right\} \leq \mathcal{D}(\lambda) \leq c_{\beta} \lambda^{\beta} \leq \left\{\begin{array}{ll}
c_\beta T^{-\beta \tilde{\theta}}, & \hbox{when} \ \theta > \frac{q+1}{q+2}, \\
c_\beta T^{-\beta \tilde{\theta}} \log T, & \hbox{when} \ \theta \leq \frac{q+1}{q+2}. \end{array}\right.
$$
Then we balance the above main part with the first term of (\ref{MainBound}) by setting
$$ \left(\frac{T^{\frac{(1-\theta)(q(2 + \zeta) + (4-2\tau +\zeta \tau))}{4}}}{m}\right)^{\frac{2}{4-2\tau+\zeta \tau}} = T^{-\beta \tilde{\theta}}. $$
This leads us to choose $T$ to be the integer part of
\be
\lceil m^{\gamma} \rceil, \ \hbox{where} \
 \gamma := \frac{2}{\left(\frac{1-\theta}{2} + \beta \tilde{\theta} \right) (4-2\tau+\zeta \tau) + \frac{q (2 + \zeta)(1-\theta)}{2}}. \label{choiceT}
\ee
With this choice, the main part of (\ref{MainBound}) can be bounded as
\begin{eqnarray*}
&& \max\left\{\left(\frac{T^{\frac{(1-\theta)(q(2 + \zeta) + (4-2\tau +\zeta \tau))}{4}}}{m}\right)^{\frac{2}{4-2\tau+\zeta \tau}}, \ R^2 \Lambda_T,\ {\mathcal A}(f_*)\right\} \\
&& \quad \leq \left\{\begin{array}{ll}
2 c_\beta m^{-\beta \tilde{\theta} \gamma}, & \hbox{when} \ \theta > \frac{q+1}{q+2}, \\
2 \gamma c_\beta m^{- \beta \tilde{\theta} \gamma } \log m, & \hbox{when} \ \theta \leq \frac{q+1}{q+2}. \end{array}\right.
\end{eqnarray*}

Notice from the definition of $\tilde{\theta}$, one can easily prove that $\tilde{\theta} \leq 1 - \theta$. Then $R /\sqrt{c_{\beta}} \leq \sqrt{ \lambda^{\beta - 1}} \leq \Lambda_T^{\beta - 1 \over 2} \leq T^{\frac{1 - \theta}{2}}$ and the restriction for $R$ in Theorem \ref{MainEstimates} is satisfied up to constants. The restriction $T^{\frac{q (1-\theta)}{2}} m^{- \frac{2}{2+\zeta}}\leq 1$ is also satisfied because
$$ T^{\frac{q (1-\theta)}{2}} \leq m^{\frac{q (1-\theta) \gamma}{2}} \leq m^{\frac{2}{2+\zeta}}. $$
Observe that $\gamma \leq \frac{2}{1-\theta}$. So by Theorem \ref{MainEstimates}, with confidence $1-\delta$, we have
$$
{\mathcal E}(f_T) - {\mathcal E}(f_\rho^V)
\leq \left\{\begin{array}{ll}
2 c_\beta \widetilde{C}_3 m^{-\beta \tilde{\theta} \gamma} \log \frac{2}{\delta}, & \hbox{when} \ \theta > \frac{q+1}{q+2}, \\
\frac{4 c_\beta \widetilde{C}_3}{1-\theta} m^{-\beta \tilde{\theta} \gamma} \log m  \log \frac{2}{\delta}, & \hbox{when} \ \theta \leq \frac{q+1}{q+2}. \end{array}\right.
$$
Observe that the power index $\beta \tilde{\theta} \gamma$ is
$$ \beta \tilde{\theta} \gamma = \left\{\begin{array}{ll}
\frac{\beta}{\beta (2 -\tau + \zeta \tau/2) + \left\{\frac{2 -\tau + \zeta \tau/2}{2} + \frac{q(1+ \zeta/2)}{2}\right\}}, & \hbox{when} \ \theta \geq \frac{q+1}{q+2}, \\
\frac{\beta}{\beta (2 -\tau + \zeta \tau/2) +  \frac{1-\theta}{\theta(1+q) -q} \left\{\frac{2 -\tau + \zeta \tau/2}{2} + \frac{q(1+ \zeta/2)}{2}\right\}}, & \hbox{when} \ \theta < \frac{q+1}{q+2}, \end{array}\right.
$$
while the index $\gamma$ can be expressed by (\ref{indexgamma}).
Then our desired learning rates are verified by setting the constant $\widetilde{C} =2 c_\beta  \widetilde{C}_3$ when $\theta > \frac{q+1}{q+2}$ while
$\widetilde{C} =\frac{4 c_\beta \widetilde{C}_3}{1-\theta}$ when $\theta \leq \frac{q+1}{q+2}$. The proof of Theorem \ref{MainRates} is complete.
\qed

\begin{proof}[Proof of Theorem \ref{MainRatesAverge}]
We only sketch the proof for the case $g_T = a_T$.
It is easy to prove the following upper bound for $\|a_T\|_K$ by applying Lemma \ref{Lem4}:
$$\|a_T\|_K \leq T^{1-\theta \over 2}.$$
With the upper bound on $a_T$ and Lemma \ref{empexcessAverge}, a similar argument as that for Theorem \ref{MainRates}, one can prove the results. We omit the details.
\end{proof}

\begin{proof}
  [Proof of Theorem \ref{MainRatesSmoothLoss}]
  With lemmas \ref{ERMDifference}, \ref{BoundSmoothLoss}, and a similar approach as that for Theorem \ref{MainRates}, we can prove the convergence results for smooth loss functions. We omit the details.
\end{proof}

\proof[Proof of Theorem \ref{hingeRates1}] We use Theorem \ref{MainRates} to
prove the results. The hinge loss satisfies (\ref{EqCond1}) with $q = 0$ and $c_q = \frac{1}{2}$, $|V|_0 =1$ and $\|f_{\rho}^V\|_{\infty}=1$ where $f_{\rho}^V$ is the Bayes rule $f_c$. Condition (\ref{varianceexpect}) is valid with
$\tau =0$ and $c_\tau =1$. Since $\theta>1/2,$ by simple
calculations, one finds that
$\gamma=\frac{1}{(1-\theta)(2\beta+1)}$ and
$\alpha=\frac{\beta}{2\beta+1}.$

Using the comparison theorem from \cite{Zhang2004}, we have
$${\mathcal R}\left(\hbox{sign}(f_T)\right) - {\mathcal R}(f_c)
\leq {\mathcal E}(f_T) - {\mathcal E}(f_\rho^V).
$$
So the desired probabilistic upper bound (\ref{HingeLossBound}) for the hinge loss follows from
the above inequality and Theorem \ref{MainRates}.

It remains to prove the second part of the theorem.
Since $0< \epsilon < \frac{1}{3}$, the restriction $\beta > \frac{1 - 3\epsilon}{1 + 6 \epsilon} $ for the approximation order tells us that the index
$$\alpha = {\beta \over 2\beta + 1} = {1 \over 2 + 1/\beta} \geq {1 \over 3} - \epsilon.$$
The proof of Theorem \ref{hingeRates1} is complete. \qed

\proof[Proof of Theorem \ref{hingeRates}]
Since $0< \epsilon < \frac{1}{3}$, the restriction $\beta > \frac{4 - 3\epsilon}{4 + 6 \epsilon} $ for the approximation order tells us that the parameter $\theta$ satisfies $\frac{1}{2}< \theta < 1$ and the index
$$\gamma = {1 \over (1-\theta) (2\beta + 1)} = \frac{2}{3} + \epsilon.$$ Finally we find that the index
$$ \alpha =  \frac{\beta}{2 \beta + 1} = \frac{1}{2 + 1/\beta} \geq \frac{1}{3}- \frac{\epsilon}{4}. $$
So the desired probabilistic upper bound follows from the first conclusion of Theorem \ref{hingeRates1}. The proof of Theorem \ref{hingeRates} is complete.
\qed

\section{Conclusions}
This paper proposes and studies iterative regularization approaches for learning with convex loss functions. More precisely, we study how regularization can be achieved by early stopping an empirical iteration induced  by the subgradient method, or gradient descent in the case the loss is also smooth. Finite sample bounds are established providing indications on how to suitably choose the  step-size and the stopping rule. Differently to classical results on the subgradient method,we analyze the behavior of the last iterate showing it has essentially the same properties of  the average, to the best, iterate. These results provide a theoretical foundation for early stopping with convex losses.

Beyond the analysis in the paper our error decomposition provides an approach to incorporate
statistical and optimization aspects in the analysis of learning algorithms. While a natural development will be to sharpen the bounds and perform extensive empirical tests, we hope the study in the paper can help deriving novel and faster algorithms, for example analyzing accelerations \cite{nest04}, or distributed approaches, within the framework we propose.

\section*{Acknowledgments}
The work described in this paper is supported partially by the Research Grants Council
of Hong Kong [Project No. CityU 11304114] and by National Natural Science Foundation of China under Grant 11461161006. LR is supported by the FIRB project RBFR12M3AC ÒLearning meets time: a new computational approach for learning in dynamic systems?and the Center for Minds, Brains and Machines (CBMM), funded by NSF STC award CCF-1231216.

\appendix

\bibliographystyle{unsrt}

\end{document}